\documentclass{colt2017-arxiv} 

\usepackage{times}
\usepackage{hyperref}
\usepackage{mathtools}
\usepackage{multirow,color,graphicx}
\usepackage{subfloat}
\usepackage{xfrac,amsfonts,amsbsy}
\usepackage[titletoc]{appendix}
\usepackage{xspace}
\usepackage{savesym,verbatim}
\usepackage[titletoc]{appendix}
\usepackage{paralist}
\usepackage{subscript}
\usepackage{tikz}

\makeatletter

\newcommand{\Rmnum}[1]{\expandafter\@slowromancap\romannumeral #1@}
\makeatother

\DeclareMathOperator*{\argmax}{arg\,max}

\renewcommand{\algorithmcfname}{Algorithm}
\newcommand{\R}{\mathbb{R}}
\newcommand{\E}{\mathbb{E}}

\newcommand{\indfunc}{\ensuremath{\mathbf{1}}}

\newcommand{\bigObound}{\ensuremath{\mathcal{O}}}
\newcommand{\smallObound}{\ensuremath{o}}
\newcommand{\Thetabound}{\ensuremath{\Theta}}

\newcommand{\hedge}{\textsc{Hedge}\xspace}
\newcommand{\expthree}{\textsc{EXP3}\xspace}
\newcommand{\expfour}{\textsc{EXP4}\xspace}

\newcommand{\algo}{\textsc{LearnExp}\xspace}
\newcommand{\generalalgo}{\textsc{Algo}\xspace}
\newcommand{\adv}{\textsc{Adv}\xspace}
\newcommand{\expert}{\textsc{Exp}\xspace}

\newcommand{\num}{\ensuremath{N}}

\newcommand{\actionSet}{\ensuremath{\mathcal{A}}}
\newcommand{\action}{\ensuremath{a}}

\newcommand{\contextSet}{\ensuremath{\mathcal{X}}}
\newcommand{\context}{\ensuremath{x}}

\newcommand{\paramSet}{\ensuremath{S}}
\newcommand{\param}{\ensuremath{w}}

\newcommand{\regretRate}{\ensuremath{\beta}}
\newcommand{\smoothRate}{\ensuremath{\alpha}}

\newcommand{\loss}{\ensuremath{l}}
\newcommand{\feedback}{\ensuremath{f}}

\newcommand{\val}{\ensuremath{V}}
\newcommand{\Time}{\ensuremath{T}}
\newcommand{\omdT}{\ensuremath{T}}

\newcommand{\lossmax}{\ensuremath{l_{max}}}

\newcommand{\sequence}{\ensuremath{\mathcal{S}}}
\newcommand{\sequencelen}{\ensuremath{|\mathcal{S}|}}

\newcommand{\historySet}{\ensuremath{\mathcal{H}}}
\newcommand{\history}{\ensuremath{h}}

\newcommand{\regret}{\textsc{Reg}\xspace}

\newcommand{\norm}[1]{\left\lVert#1\right\rVert}

\title[Learning to Use Learners' Advice]{Learning to Use Learners' Advice}

\coltauthor{\Name{Adish Singla} \Email{adish.singla@inf.ethz.ch}\\
\Name{Hamed Hassani} \Email{hamed@inf.ethz.ch}\\
\Name{Andreas Krause} \Email{krausea@ethz.ch}\\
\addr ETH Zurich, Zurich, Switzerland}

\begin{document}
\maketitle


\begin{abstract}
In this paper, we study a variant of the framework of online learning using expert advice with limited/bandit feedback. We consider each expert as a learning entity, seeking to more accurately reflecting certain real-world applications. In our setting, the feedback at any time $t$ is limited in a sense that it is only available to the expert $i^t$ that has been selected by the central algorithm (forecaster), \emph{i.e.}, only the expert $i^t$ receives feedback from the environment and gets to learn at time $t$. We consider a generic black-box approach whereby the forecaster does not control or know the learning dynamics of the experts apart from knowing the following no-regret learning property:  the average regret of any expert $j$ vanishes at a rate of at least $\bigObound(t_j^{\regretRate-1})$ with $t_j$ learning steps where $\regretRate \in [0, 1]$ is a parameter.  

In the spirit of competing against the best action in hindsight in multi-armed bandits problem, our goal here is to be competitive w.r.t. the cumulative losses the algorithm could receive by following the policy of always selecting one expert. We prove the following hardness result:  without any coordination between the forecaster and the experts, it is impossible to design a forecaster achieving no-regret guarantees. In order to circumvent this hardness result, we consider a practical assumption allowing the forecaster to ``guide" the learning process of the experts by filtering/blocking some of the feedbacks observed by them from the environment, \emph{i.e.}, not allowing the selected expert $i^t$ to learn at time $t$ for some time steps. Then, we design a novel no-regret learning algorithm \algo for this problem setting by carefully guiding the feedbacks observed by experts. We prove that \algo achieves the worst-case expected cumulative regret of $\bigObound(\Time^\frac{1}{2 - \regretRate})$ after $\Time$ time steps and matches the regret bound of $\Thetabound(\Time^\frac{1}{2})$ for the special case of multi-armed bandits.
\end{abstract}

\vspace{-2mm}
\section{Introduction}\label{sec.introduction}
Many real-world applications involve repeatedly making decisions under uncertainty---for instance, choosing one of the several items to recommend to the user, dynamically allocating resources among available stock options in a financial market, or sequentially deciding the next medical test in healthcare.  Furthermore, the feedback is often limited in these settings in a sense that only the loss/reward associated with the action taken by the system is observed, referred to as the bandit feedback setting. Online learning using expert advice with bandit/limited feedback is a well-studied framework to model the above-mentioned application settings \citep{freund1995desicion,auer2002nonstochastic,cesa2006prediction,bubeck2012regret} and addresses the fundamental question of how a learning algorithm should trade-off exploration (the cost of acquiring new information) versus exploitation (acting greedily based on current information to minimize instantaneous losses). In this paper, we investigate this framework with an important practical consideration: 
\begin{quote}
\emph{How do we use the advice of experts when they themselves are learning entities?}
\end{quote}

\subsection{Motivating Applications}
Modeling experts as learning entities realistically captures many practical scenarios of how one would define/encounter these experts in real-world applications, such as seeking advice from fellow players or friends, aggregating prediction recommendations from trading agents or different marketplaces, product testing with human participants who might adapt over time, information acquisition from crowdsourcing participants who might learn over time, the problem of meta-learning and hyperparameter tuning whereby different learning algorithms are treated as experts (\emph{cf.} \citet{baram2004online,hsu2015active}), and many more. 

As a concrete running example, we consider the problem of learning to offer personalized deals / discount coupons to users enabling new businesses to incentivize and attract more customers \citep{edelman2011groupon,singla16hemimetric}. An emerging trend is \emph{deal-aggregator} sites like \emph{Yipit}\footnote{\url{http://yipit.com/}; \url{http://www.groupon.com/}; \url{https://livingsocial.com/}\label{footnote.1}} providing personalized coupon recommendation services to their users by aggregating and selecting coupons from \emph{daily-deal} marketplaces like \emph{Groupon} and \emph{LivingSocial}\textsuperscript{\ref{footnote.1}}. One of the primary goals of these recommendation systems like \emph{Yipit} (corresponding to the central algorithm / forecaster in our setting) is to design better selection strategies for choosing coupons from different marketplaces (corresponding to the experts in our setting). However, these marketplaces (experts) themselves would be learning to optimize the coupons to offer, for instance, the discount price or the type of the coupon based on historic interactions with users \citep{edelman2011groupon}.

\subsection{Experts as Learning Entities: Challenges and Our Results}
We now provide an overview of our approach, the main challenges in designing a forecaster with no-regret guarantees, and our results. 

{\bfseries The interaction model.}
We consider an online setting similar to that of adversarial online learning using experts' advice with bandit feedback \citep{auer2002nonstochastic}. However, to keep the presentation more general (\emph{e.g.}, we do not necessarily require that the sets of actions are shared across different experts), the forecaster selects / seeks advice from only one expert $i^t$ at any time $t$ (\emph{cf.} \citet{kale2014multiarmed} for more discussion about this aspect). More specifically, at time $t$, the forecaster selects an expert $i^t$, performs an action $a^t_{i^t}$ recommended by the expert $i^t$, and incurs a loss $l^t(a^t_{i^t})$ set by the adversary.

{\bfseries The notion of regret.}
In the standard framework, \emph{i.e.}, when the experts are not learning entities, the \expthree algorithm \citep{auer2002nonstochastic} is well-suited for this problem setting achieving the optimal regret bounds. However, it is important to note that the classical notion of \emph{external} regret used in the literature (\emph{cf.} \citet{auer2002nonstochastic, cesa2006prediction, bubeck2012regret}) does not provide any meaningful guarantees in our setting in terms of competing against the ``best" expert. Similar to the notion of competing against the best action in hindsight in multi-armed bandits problem, we want to be competitive w.r.t. the cumulative losses the algorithm could receive by following the policy of always selecting one expert (\emph{cf.} Section~\ref{sec.model.objective} for a formal definition).

{\bfseries Experts as no-regret learners and blackbox approach.} In our setting, the experts themselves are learning entities. Formally, we assume that the experts are no-regret learners, \emph{i.e.}, the average regret of any expert $j$ vanishes at a rate of \emph{at least} $\bigObound(t_j^{\regretRate-1})$ with $t_j$ learning steps where $\regretRate \in [0, 1]$ is a parameter known to the forecaster.
We consider the following natural notion of bandit/limited feedback: only the selected expert $i^t$ receives feedback and gets to learn at time $t$; all other experts that have not been selected at time $t$ experience no change in their learning state at this time. We consider a generic black-box approach in which the forecaster does not know and cannot control the internal learning dynamics of the experts.

{\bfseries Challenges and hardness result.}
It turns out that modeling these experts as learning entities leads to a challenging twist in this well-studied and foundational online learning framework. In this paper, we prove the following hardness result for our problem setting:  without any coordination between the forecaster and the experts, it is impossible to design a forecaster achieving no-regret guarantees in the worst-case. Somewhat surprisingly, this hardness result holds when playing against an oblivious (non-adaptive) adversary and even if restricting the experts to be implementing some well-studied online learning algorithms, for instance, the \hedge algorithm \citep{freund1995desicion}. The fundamental challenge leading to this hardness result arises from the fact that the forecaster's selection strategy affects the feedback sequences observed by the experts which in turn alters their learning process.

{\bfseries ``Guided" feedbacks and achieving no-regret guarantees.}
In order to circumvent this hardness result, we consider the following practical assumption: we allow the forecaster to ``guide" the learning process of the experts by filtering/blocking some of the feedbacks observed by them from the environment, \emph{i.e.},  the selected expert $i^t$ would not learn at time $t$ for some time steps. For instance, in the motivating application of offering personalized deals to users, the deal-aggregator site (forecaster) often primarily interacts with users on behalf of the individual daily-deal marketplaces (experts) and  hence can control the flow of feedback to these marketplaces. Alternatively, we note that this process of guiding and restricting the feedback can be achieved via coordination between the forecaster and the selected expert $i^t$ with a $1$-bit of communication at time $t$. Given this additional control, we design a novel algorithm \algo for the forecaster which carefully guides the feedbacks observed by experts. We prove that \algo achieves the worst-case expected cumulative regret of $\bigObound(\Time^\frac{1}{2 - \regretRate})$ after $T$ time steps against an oblivious adversary for a rich family of no-regret learning algorithms that experts may be implementing. For the special case of multi-armed bandits,  algorithm \algo is equivalent to that of the well-studied \expthree algorithm and hence matches the optimal regret bound of $\Thetabound(\Time^\frac{1}{2})$.

{\bfseries Connections to the existing results.}
\cite{maillard2011adaptive} studied the problem of competing against an adaptive adversary when the adversary's reward generation policy is restricted to a pre-specified set of known models. For this problem, the authors introduced the \expfour/\expthree algorithm, \emph{i.e.},  \expfour meta-algorithm with experts executing \expthree algorithms proving a regret of $\bigObound(T^{\frac{2}{3}})$ (\emph{cf.} \citet{bubeck2012regret} for a variant of the algorithm). This \expfour/\expthree algorithm is perhaps closest to ours, as it involves a forecaster where the experts are the learning entities. However, we note that our hardness result does not contradict their regret bounds---the key difference in their setting is that the forecaster has the power to modify the losses as seen by experts, and it provides an unbiased estimate of the losses to these experts. Moreover, their analysis is specific to the experts implementing the \expthree or bandit algorithms, whereas the focus of this paper is to present a more generic learning framework in which experts as learning entities may implement a broad class of learning algorithms. Our work is also related to contemporary work by \citet{AgarwalLNS16}, who study a variant of the problem tackled by \cite{maillard2011adaptive} and also prove a hardness result similar to that of ours. Note that, in applications where the experts  directly receive feedback from the environment, implementing the strategies of \cite{maillard2011adaptive,AgarwalLNS16} would require the forecaster to communicate the probability $p$ with which the expert $i^t$ was selected at time $t$. However, our proposed idea of guiding the feedback can be achieved via coordination between the forecaster and the selected expert $i^t$ with a $1$-bit of communication at time $t$.
\vspace{-3mm}
\section{The Model}\label{sec.model}
\vspace{-2mm}
We have the following entities in our problem setting: (i) an algorithm \generalalgo as the forecaster; (ii) the adversary \adv acting on behalf of the environment; and (iii) $\num$ experts $\expert_j \ \forall j \in \{1, \ldots N\}$ (henceforth denoted as $[\num]$). 


\renewcommand{\algorithmcfname}{Protocol}
\begin{algorithm2e}[t!]
	\ForEach {$t = 1, 2, \ldots, \Time$}{
	    \tcc*[h]{\textcolor{blue}{Adversary generates the following}} \\
  		\nl a private loss vector  $\loss^t$, \emph{i.e.}, $\loss^t(\action) \ \forall \ \action \in \actionSet$ \\
  		\nl a private feedback vector  $\feedback^t$, \emph{i.e.}, $\feedback^t(\action) \ \forall \ \action \in \actionSet$ \\
  		\nl a public context $\context^t \in \contextSet$\\
	    \tcc*[h]{\textcolor{blue}{Selecting an expert and performing an action}}\\
  		\nl \generalalgo selects an expert $i^t \in [\num]$ denoted as $\expert_{i^t}$\\
		\nl \generalalgo performs the action $\action_{i^t}^t$ recommended by $\expert_{i^t}$\\
	    \tcc*[h]{\textcolor{blue}{Feedback and updates}} \\
		\nl \generalalgo incurs (and observes) loss $\loss^t(\action_{i^t}^t)$ and updates its selection strategy	\\
		\nl	$ \forall j \in [\num]: j \neq i^t$, $\expert_j$ does not observe any feedback and makes no update\\			
		\nl $\expert_{i^t}$  observes feedback $\feedback^t(\action_{i^t}^t)$ from the environment and updates its learning state\\	
  }
  \caption{The interaction between adversary \adv, algorithm \generalalgo, and experts}
  \label{interaction} 
\end{algorithm2e}

\subsection{Specification of the Interaction}
Protocol~\ref{interaction} provides a high-level specification of the interaction between the $\num+2$ entities. The sequential decision making process proceeds in rounds $t = 1, 2, \ldots, \Time$ (henceforth denoted as $[T]$); for simplicity we assume that $T$ is known in advance to the algorithm and the results in this paper can be extended to an unknown horizon via the usual doubling trick \citep{cesa2006prediction}. Each expert $\expert_j$ where $j \in \num$ is associated with a set of actions $\actionSet_j$ and the action set of the algorithm \generalalgo is given by $\actionSet = \cup_{j \in [\num]} \actionSet_j$. For the clarify of presentation in defining the loss and feedback vectors, we will consider that the action sets of experts are disjoint.\footnote{Note that assuming the disjoint action sets across experts is w.l.o.g., as we can still simulate the shared actions by enforcing a constraint that the losses generated by the adversary are same for the shared actions at any given time.}

At any time $t$, the adversary \adv generates a private loss vector $\loss^t$ (\emph{i.e.}, $\loss^t(\action) \ \forall \ \action \in \actionSet$) and a private feedback vector $\feedback^t$ (\emph{i.e.}, $\feedback^t(\action) \ \forall \ \action \in \actionSet$). Additionally, the adversary \adv generates a context $\context^t \in \contextSet$ that is accessible to all the experts while recommending their actions at time $t$---this context essentially encodes all the side information from the environment accessible to the experts at time $t$ (\emph{e.g.}, this context could represent preferences of a user arriving at time $t$ in an online recommendation system). Simultaneously, the algorithm \generalalgo (possibly with some randomization) selects expert $\expert_{i^t}$ to seek advice. The selected expert $\expert_{i^t}$ recommends an action $\action_{i^t} \in \actionSet_{i^t} \subseteq \actionSet$ (possibly with its internal randomization) which is then performed by the algorithm. As feedback, the algorithm \generalalgo observes the loss $\loss^t(\action_{i^t}^t)$ and updates its strategy on how to select experts in the future.  All the experts apart from the one selected (\emph{i.e.}, $\expert_{j} \ \forall \ j \neq i^t$) observe no feedback and make no update at this time. The selected expert $\expert_{i^t}$ observes a feedback from the environment denoted as $\feedback^t(\action_{i^t}^t)$ and updates its learning state. At the end of time $t$, the algorithm \generalalgo incurs a loss of $\loss^t(\action_{i^t}^t)$.  

So far, we have considered a generic notion of the feedback received by the selected expert---this feedback essentially depends on the application setting and is supposed to be ``compatible" with the learning algorithm used by an expert. As a concrete example, consider an expert $\expert_j$ implementing the \expthree algorithm and taking action $\action^t_j$ at time $t$, then the feedback $\feedback^t(\action^t_j)$ received by this expert (if selected at time $t$) is the loss $\loss^t(\action^t_j)$; for the case of expert $\expert_j$ implementing the \hedge algorithm, the feedback $\feedback^t(\action^t_j)$ received by this expert (if selected at time $t$) is the set of losses $\loss^t(\action) \ \forall \ \action \in \actionSet_j$. The feedback could be more general, for instance, receiving a binary signal of rejection or acceptance of the offered deal when an expert is implementing a dynamic pricing based algorithm via the partial monitoring framework \citep{cesa2006prediction,bartok2014partial}. Also, we note that the special case of standard multi-armed bandits is captured by the setting in which $\actionSet_j$ is a singleton for every expert $j \in [\num]$. 

We assume that the losses are bounded in the range $[0,\lossmax]$ for some known $\lossmax \in \R_{+}$; w.l.o.g. we will use $\lossmax = 1$ \citep{auer2002nonstochastic}. We consider an oblivious (non-adaptive adversary) as is usual in the literature \citep{freund1995desicion,auer2002nonstochastic}, \emph{i.e.}, the loss vector $\loss^t$, the feedback vector $\feedback^t$, and the context $\context^t$ at any time $t$ do not depend on the actions taken by \generalalgo, and hence can be considered to be fixed in advance. Apart from that, no other restrictions are put on the adversary, and
it has complete knowledge about the algorithm \generalalgo and the learning dynamics of the experts. 






\subsection{Specification of the Experts}
We consider a generic black-box approach in which \generalalgo does not know and cannot control the internal dynamics of the experts. In order to formally state the objective and guarantees we seek, we now provide a generic specification of the experts. At time $t$, let us denote an instance of feedback received by $\expert_{i^t}$ by a ßtuple $\history = (\action^t_{i^t}, \context^t, \feedback^t(\action^t_{i^t}))$. For any expert $\expert_j$ where $j \in [\num]$, let $\historySet^t_j = (\history^1, \history^2, \ldots)$ denote the feedback history for  $\expert_j$, \emph{i.e.}, an ordered sequence of feedback instances observed by $\expert_j$ up to time $t$. The length $|\historySet^t_j|$ denotes the number of learning steps for $\expert_j$ up to time $t$. At time $t$, the action $\action^t_j$ recommended by $\expert_j$ to the algorithm, if this expert is selected,  is given by $\action^t_j = \pi_j(\context^t, \historySet^t_j)$  where $\pi_j$ is a (possibly randomized) function of $\expert_j$, taking as input a context and a history of feedback sequence, and outputs an action $\action \in \actionSet_j$. Importantly, this history $\historySet^t_j$ is dependent on the execution of the algorithm \generalalgo--- for clarify of presentation, we denote it as $\historySet^t_{j, \generalalgo}$.

{\bfseries No-regret learning dynamics.} To be able to say anything meaningful in this setting, we introduce the constraint of \emph{no-regret} learning dynamics on the experts.\footnote{In order to prove the no-regret guarantees for our algorithm \algo in Section~\ref{sec.alg}, this constraint is required to hold \emph{only} for the \emph{best} expert against which we want to be competitive, a less stringent requirement.}  Let us consider any sequence of loss vector $\loss$, feedback vector $\feedback$, and context $x$ given by $\sequence = \big((\loss^\tau, \feedback^\tau, \context^\tau)\big)_{\tau=\{1, 2, \ldots\}}$ generated arbitrarily by \adv and let $\sequencelen$ denotes its length. Consider a setting in which an expert $\expert_j$ for any $j \in [\num]$ is selected at every time step. At every time step $\tau \in [\sequencelen]$, $\expert_j$ recommends an action $\action^\tau_j$, accumulates the loss $\loss(\action^\tau_j)$, and observes the feedback $\feedback(\action^\tau_j)$. In this setting, $\expert_j$ observes feedback at every time step and we denote this ``complete" history of feedback sequence at any time $\tau \in \sequencelen$ as  $\historySet^t_{j,1}$ whereby $1$ denotes the fact that this expert is selected and receives feedback with probability $1$ at every time step. Then, the no-regret learning dynamics of $\expert_j$ parameterized by $\regretRate_j \in [0, 1]$ guarantees that the expected average regret vanishes as follows\footnote{Note that this is a weaker notion of regret---any deterministic policy $\pi_j$ that always outputs a constant action has $\beta_j=0$. However, $\beta_j=0$ would be the right way to characterize the learning dynamics of this expert for our setting.}:
\begin{align}
\E\bigg[\frac{1}{\sequencelen}\sum_{\tau=1}^{\sequencelen} \loss^\tau\big(\pi_j(\context^\tau, \historySet^\tau_{j,1})\big)\bigg] - \E\bigg[\frac{1}{\sequencelen}\sum_{\tau=1}^{\sequencelen} \loss^\tau\big(\pi_j(\context^\tau, \historySet^{\sequencelen}_{j,1})\big)\bigg] \leq \bigObound(\sequencelen^{\regretRate_j - 1}) 
\label{eq.noregretdynamics}
\end{align}
where the expectation is w.r.t. the randomization of function $\pi_j$. We assume that  parameter $\regretRate \in [0, 1]$ upper bounds the regret rate parameters of individual experts and is a parameter known to the forecaster.\footnote{Again for our algorithm \algo in Section~\ref{sec.alg}, $\regretRate$ \emph{only} needs to upper bound the regret rate for the \emph{best} expert against which we want to be competitive.}

\subsection{Our Objective: No-Regret Guarantees}\label{sec.model.objective}
Intuitively, we want to be competitive w.r.t. the cumulative losses the algorithm could receive by following the policy of always using the advice of one single expert---such a policy ensures that the single expert gets more feedback to improve its learning state and hence incur less cumulative loss. This is a challenging problem when the experts are learning entities. For instance, what may go wrong is that the \emph{best} expert could have a slow rate of learning/convergence thus incurring high losses in the beginning, misleading the algorithm to essentially ``downweigh" this expert. This is turn further exacerbates the problem for the best expert in the bandit feedback setting as this expert will be selected less and will have fewer learning steps to improve its state. This adds new challenges to the classic trade-off between exploration and exploitation, suggesting the need to explore at higher rate to tackle this problem.

Let us begin by looking at the classic notion of \emph{external} regret used in the literature \citep{auer2002nonstochastic, cesa2006prediction, bubeck2012regret}. Given that the experts are learning entities, naturally the losses incurred at any time step are dependent on the history of the forecaster's actions as that history defines the current learning state of the individual experts. Given this subtle issue of history dependent losses, the usual notion of \emph{external} regret does not provide any meaningful guarantees in terms of competing against the ``best expert in hindsight" (see below for a formal definition); the bounds given by the \emph{external} regret are only w.r.t. the post hoc sequence of actions performed and losses observed during the execution of the algorithm (\emph{cf.} \citet{maillard2011adaptive,arora2012online,mcmahan2009tighter} for more discussion on this).

We consider the following natural notion of regret in this paper: our goal is to be competitive w.r.t.  the \emph{best expert in hindsight}, that is, competitive w.r.t. the  cumulative loss that any expert could have received with the optimal actions it could have taken in hindsight. Formally, the expected cumulative regret of \generalalgo against the best expert in hindsight is given by:

\vspace{-2mm}
\begin{align}
\regret(\Time, \generalalgo) \coloneqq \sum_{t=1}^{\Time}\E\bigg[\loss^t\Big(\pi_{i^t}(\context^t, \historySet^t_{{i^t},\generalalgo})\Big)\bigg] - \min_{j \in [\num]} \E\bigg[\sum_{t=1}^{\Time} \loss^t\Big(\pi_j(\context^t, \historySet^{T}_{j, 1})\Big)\bigg] \label{eq.objective}
\end{align}
where the expectation is w.r.t. the randomization of the algorithm as well as any internal randomization of the experts. Our goal is to design an algorithm \generalalgo for the forecaster so that the regret $\regret(\Time, \generalalgo)$ grows sublinearly in time $T$.

\section{Hardness Result}\label{sec.hardness}
We show in this section that, in the absence of any coordination between the forecaster and experts, it is impossible to design a forecaster that achieves no-regret guarantees in the worst-case. Somewhat surprisingly, we prove this hardness result when playing against an oblivious (non-adaptive) adversary and when restricting the experts to be implementing the well-studied \hedge algorithm \citep{freund1995desicion}. We formally state this hardness result in the Theorem~\ref{hardness_result} below.



\begin{theorem} \label{hardness_result}
There is a setting in which each of the experts has no-regret learning dynamics with parameter $\regretRate = \frac{1}{2}$; however, any algorithm \generalalgo (forecaster) will suffer a positive average regret, i.e., $\regret(\Time, \generalalgo) = \Omega(T)$. 
\end{theorem}
\begin{figure*}[!t]
\centering
   \subfigure[Cumulative losses ($L_1$)]{
     \includegraphics[width=0.29\textwidth]{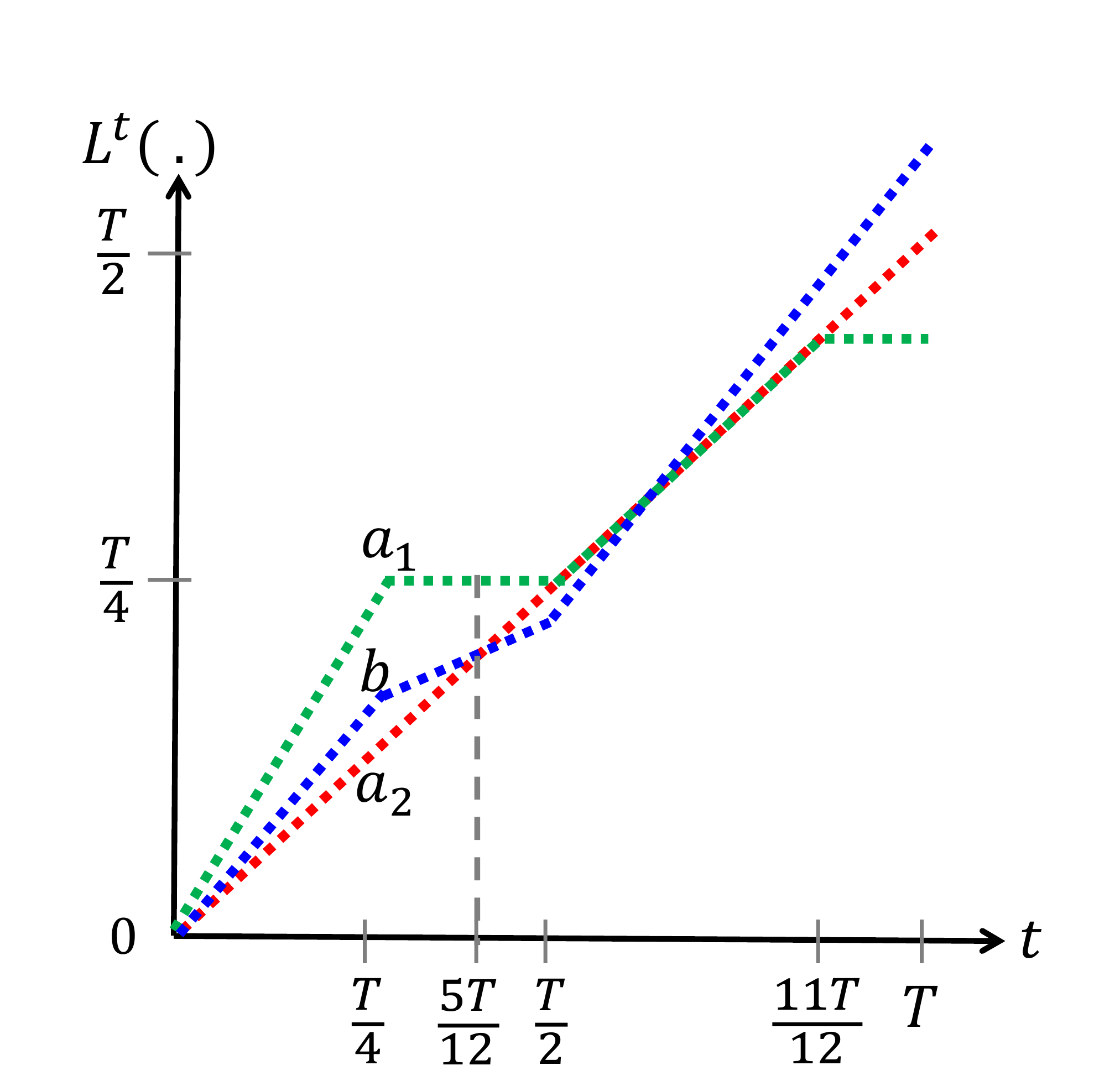}
     \label{fig.hardness-example-L1}
   }
   \subfigure[Cumulative losses ($L_2$)]{
     \includegraphics[width=0.29\textwidth]{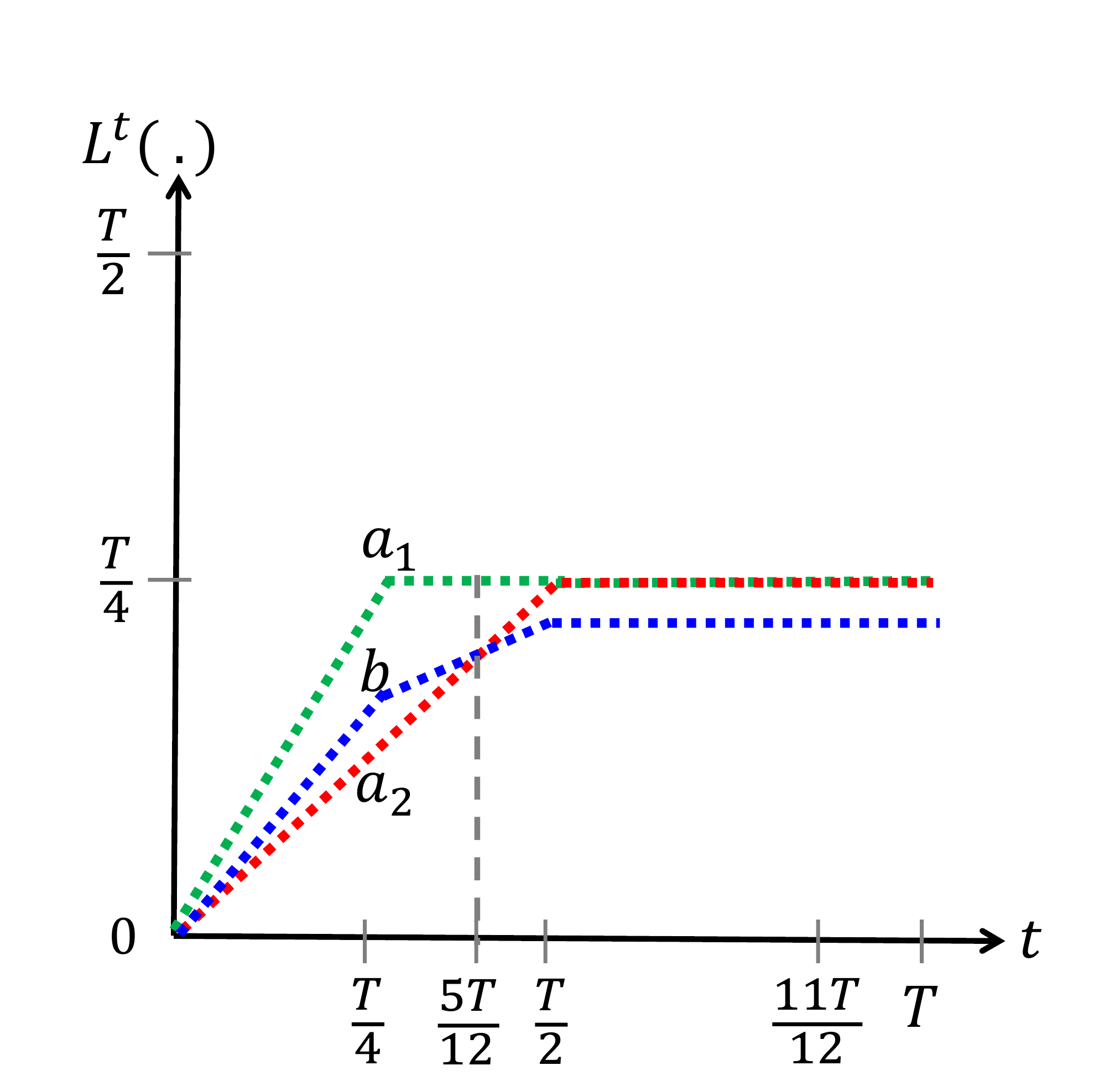}
     \label{fig.hardness-example-L2}
   }
   \subfigure[Cumulative losses ($L_3$)]{
     \includegraphics[width=0.29\textwidth]{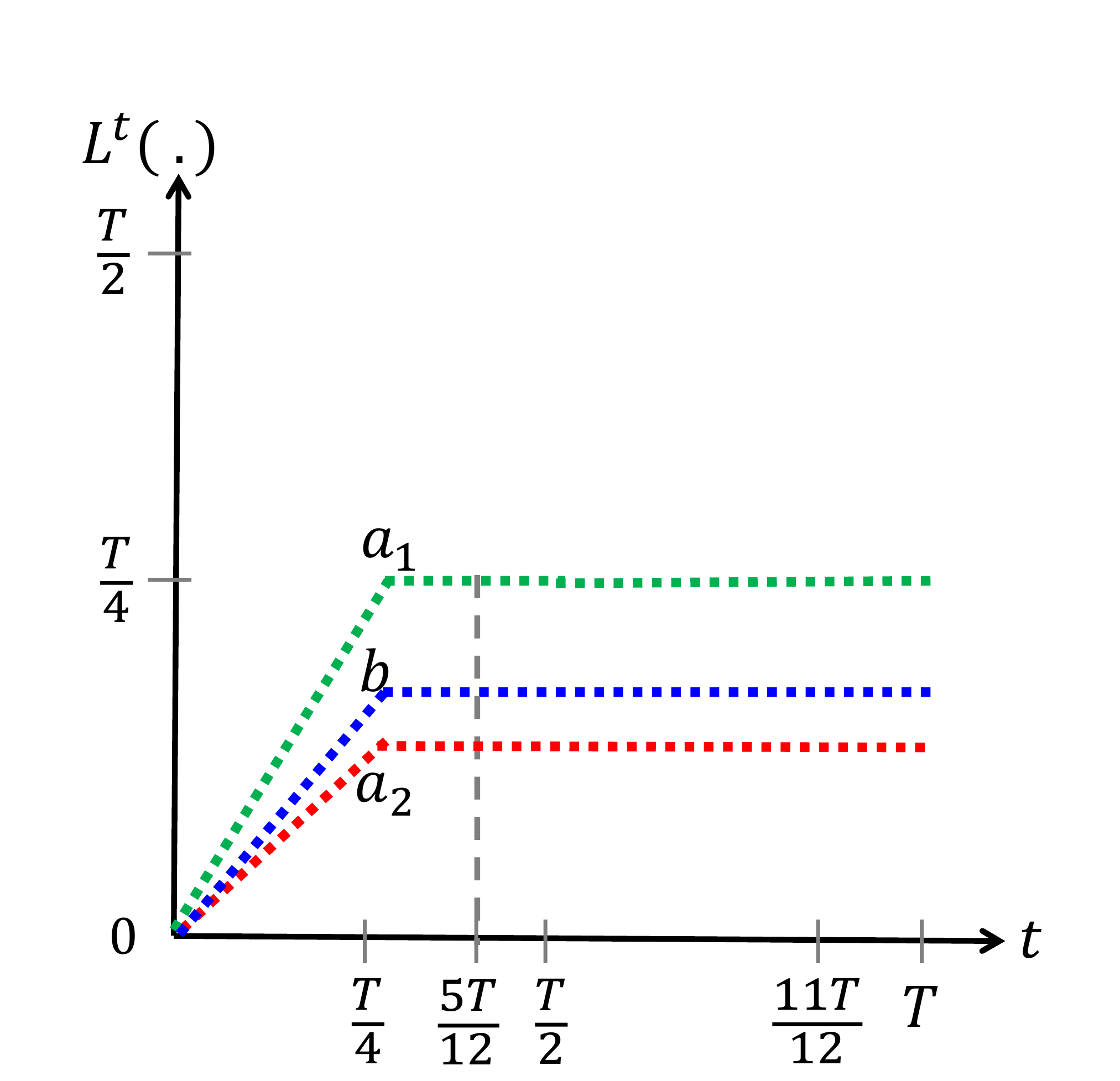}
     \label{fig.hardness-example-L3}
   }
\caption{We have two experts: $\expert_1$ plays \hedge and has two actions $\actionSet_1 = \{a_1, a_2\}$, and $\expert_2$ has only one action $\actionSet_2 = \{b\}$. Figures~\ref{fig.hardness-example-L1}, ~\ref{fig.hardness-example-L2}, and ~\ref{fig.hardness-example-L3} shows the \textit{cumulative} loss sequences $L_1$, $L_2$, and $L_3$ for three different scenarios---the adversary at $t=0$ uniformly at random picks one of these scenarios and uses that loss sequence. These plots show the cumulative losses of the three actions $\actionSet=\{a_1, a_2, b\}$ for three different sequences. The losses are illustrated with the following color scheme---$a_1$:\emph{green}, $a_2$:\emph{red}, and $b$:\emph{blue}.}
\label{fig.hardness-example-three-sequences}
\end{figure*}
%

The proof is given in the Appendix, we briefly outline the main ideas below. Our setting for proving this theorem consists of two experts $\expert_1$ and $\expert_2$.  The first expert $\expert_1$ has two actions given by $\actionSet_1 = \{a_1, a_2\}$, and the second expert $\expert_1$ has only one action given by $\actionSet_2 = \{b\}$. The action set of the algorithm \generalalgo is given by $\actionSet = \{a_1, a_2, b\}$. The expert $\expert_1$  plays the \hedge algorithm \citep{freund1995desicion}, \emph{i.e.}, the regret rate parameter is $\beta_1 = 0.5$; the expert $\expert_2$ has only one action to play as in the standard multi-armed bandit with $\beta_2 = 0$.\footnote{In fact, this hardness result holds even when considering a powerful forecaster which knows exactly the learning algorithms used by the experts, and is able to see the losses $\{l^t(a_1), l^t(a_2), l^t(a_3)\}$ at every time $t \in [T]$.} Figures~\ref{fig.hardness-example-L1}, ~\ref{fig.hardness-example-L2}, and ~\ref{fig.hardness-example-L3} show the \textit{cumulative} loss sequences $L_1$, $L_2$, and $L_3$ for three different scenarios---the adversary at $t=0$ uniformly at random picks one of these scenarios and uses that loss sequence.


The main idea of the proof uses the following arguments. We consider the case where the forecaster is facing the sequence $L_1$ (chosen by the adversary with probability $\frac{1}{3}$ at $t=0$). We then divide the time horizon $T$ into different slots and discuss the execution behavior of the forecaster and experts over these time slots. Specifically, our claim is that in the time slot $t \in (\frac{T}{4}, \frac{T}{2}]$, the expert $\expert_1$ would not be selected for $\frac{T}{12}  - \smallObound(T))$ time steps. As a result, in the time slot $t \in (\frac{11T}{12}, T]$, the expert $\expert_1$ would select action $a_2$ almost surely, and $a_1$ would only be selected $\smallObound(T)$ number of times, leading to a positive average regret for the forecaster. Informally speaking, our negative example shows that the forecaster's selection strategy could add ``blind spots" in the feedback history seen by the experts and that they might not be able to ``recover" from this. The key fundamental challenge leading to this hardness result is that the forecaster's selection strategy affects the feedback sequences observed by the experts, which in turn alters the experts' learning process.

\section{Our Algorithm \algo} \label{sec.alg}
In this section, we introduce a practical assumption that allows the forecaster to ``guide" the learning process of experts, and then we design our main algorithm \algo with provable no-regret guarantees.

\subsection{Guided Feedbacks}
In order to circumvent the hardness result proved in Section~\ref{sec.hardness}, we now consider a practical assumption motivated by the application setting of deal-aggregator sites, as discussed in Section~\ref{sec.introduction}. Usually, a deal-aggregator site interacts with users on behalf of the individual daily-deal marketplaces (experts) and hence could control the flow of feedback to these marketplaces. Hence, we allow the forecaster to ``guide" the learning process of the experts by filtering/blocking some of the feedback the experts receive from the environment.  Recall that at time $t$, as per the interaction model presented in Section~\ref{sec.model}, the selected expert $\expert_{i^t}$ observes feedback $\feedback^t(\action_{i^t}^t)$ from the environment. We now consider the setting with the following additional power in the hands of the forecaster: In order to guide the learning process of the experts, the forecaster at time $t$ could block the feedback, \emph{i.e.}, the expert $\expert_{i^t}$ would not observe feedback at time $t$  and hence would not learn at this time (just like any other experts who were not selected at time $t$). Alternatively, we note that this process of guiding the feedback could be achieved via coordination between the forecaster and the selected expert $\expert_{i^t}$ with a $1$-bit communication at time $t$.


\renewcommand{\algorithmcfname}{Algorithm}
\begin{algorithm2e}[t!]
\nl {\bf Parameters}: {$\eta \in (0, 1]$} \\
\nl {\bf Initialize}: {time $t = 1$, weights $w_j^t = 1 \ \forall j \in [\num]$}\\
	\ForEach {$t = 1, 2, \ldots, \Time$}{
	    \tcc*[h]{\textcolor{blue}{Selecting an expert and performing an action}}\\
		\nl $ \forall j \in [\num]$, define probability $p^t_j = (1 - \eta) \cdot \dfrac{w^t_j}{\big(\sum_{k \in [\num]} w^t_k\big)} + \dfrac{\eta}{\num}$\\
  		\nl Draw $i^t$ from the multinomial distribution $(p^t_j)_{j \in [\num]}$\\  	
  				
		\nl Perform action $\action_{i^t}^t$ recommended by the expert $\expert_{i^t}$\\
	    \tcc*[h]{\textcolor{blue}{Observing the loss and making updates}} \\
		\nl Observe loss $\loss^t(\action_{i^t}^t)$	\\
		\nl $ \forall j \in [\num]$, do the following: \\
			\Indp
			 \nl Set $\widetilde{\loss}^t_j$ as follows: $\widetilde{\loss}^t_j = \dfrac{\loss^t(\action_{i^t}^t)}{p^{t}_{i^t}} \text{ for } j = i^t$, else $\widetilde{\loss}^t_j = 0$ \\
			\nl	Update $w^{t+1}_j \leftarrow w^{t}_j \cdot \exp({-\dfrac{\eta \cdot \widetilde{\loss}^t_j}{\num}})$\\
			\Indm
	    \tcc*[h]{\textcolor{blue}{Guiding the feedback}} \\
		\nl $\xi^t \sim Bernoulli(\dfrac{\eta}{\num \cdot p^t_{i^t}})$ \label{algoline.guide1} \\
		\nl \If{$(\xi^t = 1)$}{ \label{algoline.guide2}
		\nl $\expert_{i^t}$  observes feedback $\feedback^t(\action_{i^t}^t)$ from the environment and updates its learning state \label{algoline.guide3}\\
		}
  }
  \caption{\algo}
  \label{mainalgo} 
\end{algorithm2e}

\subsection{Algorithm \algo}
With this additional power of the forecaster to guide feedback, we develop our main algorithm \algo, presented in Algorithm~\ref{mainalgo}. The selection strategy of the algorithm \algo is similar to the EXP family of algorithms, and in particular is equivalent to the \expthree algorithm by \cite{auer2002nonstochastic}. The core idea of guiding the feedbacks observed by experts is presented in Lines~\ref{algoline.guide1},\ref{algoline.guide2}, and \ref{algoline.guide3}. 

By default, as per the Protocol~\ref{interaction}, the selected expert $\expert_{i^t}$ always observes feedback at time $t$---for this protocol, the hardness result of Theorem~\ref{hardness_result} applies. Our algorithm \algo instead decides whether the expert $\expert_{i^t}$ should observe/use the feedback based on the outcome $\xi^t$ of a coin flip with probability $\frac{\eta}{\num \cdot p^t_{i^t}}$. By choosing this particular probability, the algorithm \algo ensures that the probability that any expert $\expert_j$ observes feedback at time $t$ is constant over time and is given by $\frac{\eta}{\num}$. The key parameter of the algorithm $\eta$ would be fixed in Theorem~\ref{thm.regretbounds}  based on the regret rate $\beta$ to achieve the desired guarantees on the regret.

The guarantees in Theorem~\ref{thm.regretbounds} mean that by adding this additional control/coordination in our model, we are able to circumvent the hardness result of Theorem~\ref{hardness_result}. Interestingly, if we consider any expert $\expert_j$ for $j \in [\num]$, the history $\historySet^t_j$ at any time $t$ under this guided feedback setting would only contain a subset of the feedback instances that it would have received without guiding (\emph{i.e.}, where $\xi^t = 1 \ \forall t \in [T]$). By carefully allowing the expert to observe a strictly smaller set of feedback instances allows us to ensure that the expert $\expert_j$ achieves low regret. Considering the example we use in the proof of Theorem~\ref{hardness_result} to show the hardness results, this means that by carefully guiding the feedback received by experts, our algorithm \algo ensures that there are no ``blind spots" in the feedback history of any expert. However, in order to achieve this, the algorithm is required to explore at a higher rate, as is evident by the value of $\eta$ in Theorem~\ref{thm.regretbounds}.

\subsection{Theoretical Guarantees}
Next, we analyze the theoretical guarantees of our algorithm \algo. One approach to doing this is to consider a particular class of no-regret learning algorithms that experts implement and prove guarantees for that class. Instead, we introduce a novel, generic notion of ``smooth" no-regret learning---our theoretical guarantees are then proven for the experts that have \emph{no-regret} and \emph{smooth} learning dynamics. Next, we introduce this notion and then discuss (see Proposition~\ref{lemma.smooth}) the class of no-regret learning algorithms that also satisfy the constraint of smooth learning dynamics.

\subsubsection{Smooth Learning Dynamics}
In our bandit feedback setting, not all the experts can observe feedback at a given time step, and hence the history of feedback instances received by any particular expert is naturally ``sparse". To formally state the behavior of the learning algorithm under this sparse feedback, we now introduce a new notion, termed \emph{smooth} learning dynamics, to complement the no-regret learning dynamics defined in \eqref{eq.noregretdynamics}. Consider the same fixed sequence $\sequence$ as used in defining \eqref{eq.noregretdynamics} and an expert $\expert_j$. However, instead of observing feedback at every time step, let's say that the expert $\expert_j$ only gets to observe the feedback sporadically at a rate of $\smoothRate \in (0,1]$---we call this an $\smoothRate$-sparse history, denoted as $\historySet^l_{j, \smoothRate}$. Then, the constraint of smooth learning dynamics ensures that the expected regret of the expert $\expert_j$ when receiving the above-mentioned sparse feedback vanishes (smoothly w.r.t. rate $\smoothRate$) as follows:

\begin{align}
\E\bigg[\frac{1}{\sequencelen}\sum_{\tau=1}^{\sequencelen} \loss^\tau\big(\pi_j(\context^\tau, \historySet^\tau_{j,\smoothRate})\big)\bigg] - \E\bigg[\frac{1}{\sequencelen}\sum_{\tau=1}^{\sequencelen} \loss^\tau\big(\pi_j(\context^\tau, \historySet^{\sequencelen}_{j,1})\big)\bigg] \leq \bigObound((\smoothRate \cdot \sequencelen)^{\regretRate_j - 1})
\label{eq.smoothdynamics}
\end{align}
where the expectation is w.r.t. the randomization of function $\pi_j$ as well as w.r.t. the randomization in generating this sparse history. The following proposition states that a rich class of online learning algorithms indeed have \emph{smooth} learning dynamics that can be used by the experts, \emph{cf.} Appendix for the proof.
\begin{proposition} \label{lemma.smooth}
A rich class of no-regret online learning algorithms based on gradient-descent style updates have smooth learning dynamics including the Online Mirror Descent family of algorithms with exact or estimated gradients \citep{shalev2011online} and Online Convex Programming via greedy projections \citep{zinkevich03online}.
\end{proposition}

\subsubsection{No-regret Guarantees of \algo}
Next, we prove the no-regret guarantees of our algorithm \algo, formally stated in Theorem~\ref{thm.regretbounds}. The following theorem (stating only the leading terms w.r.t. the $\Time$ and dropping any other constants like $\num$) provides the no-regret guarantees of \algo against the best expert in hindsight as per \eqref{eq.objective}. The proof is given in the Appendix.

\begin{theorem} \label{thm.regretbounds}
Let $\Time$ be the fixed time horizon. Consider that the best expert $j^* \in [\num]$ has no-regret smooth learning dynamics parameterized by $\regretRate_{j^*} \in [0, 1]$ and \algo is invoked with input $\regretRate \in [0, 1]$ such that $\regretRate \geq \regretRate_{j^*}$. Set parameters $\eta = \Thetabound\big(\Time^{-\frac{1 - \regretRate}{2 - \regretRate}} \cdot \num^{\frac{1 - \regretRate}{2 - \regretRate}} \cdot (\log\num)^{(\frac{1}{2} \cdot\indfunc_{\{\regretRate = 0\}})}\big)$. Then, for sufficiently large $\Time$, the worst-case expected cumulative regret of \algo against the best expert in hindsight is:
$$\regret(\Time, \algo) \leq  \bigObound\big(\Time^{\frac{1}{2 - \regretRate}} \cdot \num^{\frac{1}{2 - \regretRate}} \cdot (\log\num)^{(\frac{1}{2} \cdot\indfunc_{\{\regretRate = 0\}})} \big)$$
\end{theorem}


For the special case of multi-armed bandits (where $\regretRate = 0$), this regret bound matches the bound of $\Thetabound(\Time^\frac{1}{2})$---in fact, for this special case, our algorithm \algo is exactly equivalent to \expthree. For an important case when experts are implementing algorithms like \hedge or \expthree (where $\regretRate = \frac{1}{2}$), our algorithm \algo achieves the bound of $\bigObound(\Time^\frac{2}{3})$.
\section{Background and Related Work}\label{sec.related}
In this section, we provide an overview of the relevant literature.


\vspace{-2mm}
\subsection{Background}
We begin with a background on the framework of learning using expert advice with bandit feedback.

{\bfseries Using expert advice.}
The seminal work of \cite{1992_weighted-majority,1997-acm_how-to-use-expert-advice} initiated the study of using expert advice for prediction problems, and \cite{freund1995desicion} introduced the algorithm \hedge for the general problem of dynamically allocating resources among a set of options using expert advice. 

{\bfseries Using expert advice with bandit feedback.}
However, the feedback is often limited in these settings in a sense that only the loss/reward associated with the action taken by the system is observed, referred to as the bandit feedback setting. To tackle this, \cite{auer2002nonstochastic} extended this framework to the limited feedback setting and introduced the EXP family of algorithms (\expthree, \expfour, and its variants) for multi-armed bandits  and expert advice with bandit feedback. With limited feedback, this framework addresses the fundamental question of how a learning algorithm should trade-off \emph{exploration} versus \emph{exploitation}. This framework has been studied extensively by researchers in a variety of fields and the above mentioned algorithms provide minimax optimal no-regret guarantees---we refer the reader to \cite{bubeck2012regret} for the survey on bandit problems  and monograph by \cite{cesa2006prediction}.

Furthermore, this framework is very generic and versatile to capture many complex real-world scenarios. For instance, in the EXP family of algorithms \citep{auer2002nonstochastic,mcmahan2009tighter,beygelzimer2010contextual},  each expert could have access to an arbitrary context (\emph{e.g.}, information about user preferences) that may not be shared among experts and may not be available to the algorithm. Furthermore, no statistical assumptions are needed on the process generating context or losses/rewards over time. Consequently, this framework has been used in many diverse application settings including search engine ad placement \citep{mcmahan2009tighter},  personalized news article recommendation \citep{beygelzimer2010contextual}, packet routing in networks, \citep{awerbuch2004adaptive}, and meta-learning with different learning algorithms as the experts \citep{baram2004online,hsu2015active}.

\vspace{-2mm}
\subsection{Related Work}
Next, we review research work that is relevant to the problem studied in this paper.

{\bfseries Markovian, rested, and restless bandits.}
The seminal work of \cite{gittins1979bandit} considered Markovian bandits where each action/arm is associated with its own stochastic MDP and introduces the \emph{Gittins index} to find an optimal sequential policy. Note that the arm changes its state only when it is pulled, hence also termed as \emph{rested} bandits. \cite{whittle1988restless} considered an extension termed \emph{restless} bandits where all the arms change their reward distributions at every time step according to their associated stochastic MDP. Restless bandits are notoriously difficult to tackle (\emph{cf.} \citep{slivkins2008adapting}), thereby \cite{slivkins2008adapting} considered a type of restless bandits where the change in state is governed by a more gradual process with stochastic rewards depending upon the state. \cite{2014-nips_non-stationary-rewards} considered another type of restless bandits with stochastic reward functions, however these distributions change adversarially with a budget on the allowed variation. Our approach is similar in spirit to the rested bandits; however, none of the frameworks above would model the learning dynamics of the experts in the adversarial setting we consider.

{\bfseries Non-oblivious/adaptive adversary.}
As in our setting, the challenge of history-dependent expert rewards also arises in the case of non-oblivious/adaptive adversary \citep{maillard2011adaptive,arora2012online}. \cite{arora2012online} studied online learning with bandit feedback against a non-oblivious adversary with \emph{bounded} memory and introduced the notion of policy regret instead of the usual notion of external regret. \cite{maillard2011adaptive} studied competing against adaptive adversary when the adversary's reward generation policy is restricted to a pre-specified set of known models.  However, none of the frameworks of non-oblivious/adaptive adversary listed above model learning dynamics in our setting: It would require an adversary with unbounded memory to apply the results of \cite{arora2012online}, and an adversary with unbounded number of models to apply the techniques of \cite{maillard2011adaptive}.



{\bfseries Contextual bandits.}
Another perspective on tackling some of the applications we mentioned above is the contextual bandit framework \citep{li2010contextual,langford2007epochgreedy,agarwal2014taming}. We refer the reader to the paper by \cite{mcmahan2009tighter} for more discussion on the connection between the framework of contextual bandits and learning using expert advice with bandit feedback.

{\bfseries Learning in games.}
An orthogonal line of research studies the interaction of agents in multiplayer games where each agent uses a no-regret learning algorithm \citep{blum2007learning,syrgkanis2015fast}. The questions tackled in this line of research are very different as it focuses on the interactions of the agents, their individual as well as social utilities, and the convergence of the game to equilibrium. This orthogonal line of research reassures that the no-regret learning dynamics that we consider in this paper are indeed important and natural dynamics that are also prevalent in other application domains.
\vspace{-2mm}
\section{Conclusions}\label{sec.conclusions}
In this paper, we investigated the online learning framework using expert advice with bandit feedback with an important practical consideration: how do we use the advice of the experts when they themselves are learning entities? As our first contribution, we proved the hardness result stating that it is impossible to achieve no-regret guarantees when the experts receive feedback directly from the environment and there is no further coordination between forecaster/experts. Our hardness result sheds light on the complexity of the problem when applying this online learning framework to real-world applications whereby it is natural for experts to exhibit learning dynamics.

Then, we considered a practical assumption of ``guided" feedbacks whereby the forecaster can block/filter the feedback received by the selected expect from the environment. Under this setting, we proposed a novel algorithm \algo---we proved that \algo achieves the worst-case expected cumulative regret of $\bigObound(\Time^\frac{1}{2 - \regretRate})$ after $T$ time steps where $\regretRate$ is a parameter characterizing the individual no-regret learning dynamics of the best expert. This regret bound matches the bound of $\Thetabound(\Time^\frac{1}{2})$ for the special case of multi-armed bandits.

There are a number of research directions for future work. An interesting question to tackle is whether it is possible to design a forecaster in our setting with a worst-case cumulative regret of $\Thetabound(\Time^\frac{1}{2})$  when the individual experts have no-regret learning dynamics with $\beta = \frac{1}{2}$. In this paper, in order to circumvent the hardness result, we considered the power of blocking/filtering the feedbacks, which can equivalently be achieved with a $1$-bit of communication at every time step. An interesting direction would be to consider other practical ways of coordination and to understand the minimal coordination required to achieve no-regret guarantees.
\bibliography{refs}

\begin{thebibliography}{29}
\providecommand{\natexlab}[1]{#1}
\providecommand{\url}[1]{\texttt{#1}}
\expandafter\ifx\csname urlstyle\endcsname\relax
  \providecommand{\doi}[1]{doi: #1}\else
  \providecommand{\doi}{doi: \begingroup \urlstyle{rm}\Url}\fi

\bibitem[Agarwal et~al.(2014)Agarwal, Hsu, Kale, Langford, Li, and
  Schapire]{agarwal2014taming}
Alekh Agarwal, Daniel Hsu, Satyen Kale, John Langford, Lihong Li, and Robert
  Schapire.
\newblock Taming the monster: A fast and simple algorithm for contextual
  bandits.
\newblock In \emph{ICML}, 2014.

\bibitem[Agarwal et~al.(2016)Agarwal, Luo, Neyshabur, and
  Schapire]{AgarwalLNS16}
Alekh Agarwal, Haipeng Luo, Behnam Neyshabur, and Robert~E. Schapire.
\newblock Corralling a band of bandit algorithms.
\newblock \emph{CoRR}, abs/1612.06246, 2016.

\bibitem[Arora et~al.(2012)Arora, Dekel, and Tewari]{arora2012online}
Raman Arora, Ofer Dekel, and Ambuj Tewari.
\newblock Online bandit learning against an adaptive adversary: from regret to
  policy regret.
\newblock In \emph{ICML}, 2012.

\bibitem[Auer et~al.(2002)Auer, Cesa-Bianchi, Freund, and
  Schapire]{auer2002nonstochastic}
Peter Auer, Nicolo Cesa-Bianchi, Yoav Freund, and Robert~E Schapire.
\newblock The nonstochastic multiarmed bandit problem.
\newblock \emph{SIAM Journal on Computing}, 32\penalty0 (1):\penalty0 48--77,
  2002.

\bibitem[Awerbuch and Kleinberg(2004)]{awerbuch2004adaptive}
Baruch Awerbuch and Robert~D Kleinberg.
\newblock Adaptive routing with end-to-end feedback: Distributed learning and
  geometric approaches.
\newblock In \emph{STOC}, 2004.

\bibitem[Baram et~al.(2004)Baram, El-Yaniv, and Luz]{baram2004online}
Yoram Baram, Ran El-Yaniv, and Kobi Luz.
\newblock Online choice of active learning algorithms.
\newblock \emph{Journal of Machine Learning Research}, 2004.

\bibitem[Bart{\'o}k et~al.(2014)Bart{\'o}k, Foster, P{\'a}l, Rakhlin, and
  Szepesv{\'a}ri]{bartok2014partial}
G{\'a}bor Bart{\'o}k, Dean~P Foster, D{\'a}vid P{\'a}l, Alexander Rakhlin, and
  Csaba Szepesv{\'a}ri.
\newblock Partial monitoring -- {C}lassification, regret bounds, and
  algorithms.
\newblock \emph{Mathematics of Operations Research}, 2014.

\bibitem[Besbes et~al.(2014)Besbes, Gur, and
  Zeevi]{2014-nips_non-stationary-rewards}
Omar Besbes, Yonatan Gur, and Assaf~J. Zeevi.
\newblock Optimal exploration-exploitation in a multi-armed-bandit problem with
  non-stationary rewards.
\newblock In \emph{NIPS}, 2014.

\bibitem[Beygelzimer et~al.(2011)Beygelzimer, Langford, Li, Reyzin, and
  Schapire]{beygelzimer2010contextual}
Alina Beygelzimer, John Langford, Lihong Li, Lev Reyzin, and Robert~E Schapire.
\newblock Contextual bandit algorithms with supervised learning guarantees.
\newblock In \emph{AISTATS}, 2011.

\bibitem[Blum and Monsour(2007)]{blum2007learning}
Avrim Blum and Yishay Monsour.
\newblock Learning, regret minimization, and equilibria.
\newblock 2007.

\bibitem[Bubeck and Cesa-Bianchi(2012)]{bubeck2012regret}
S{\'e}bastien Bubeck and Nicolo Cesa-Bianchi.
\newblock Regret analysis of stochastic and nonstochastic multi-armed bandit
  problems.
\newblock \emph{Machine Learning}, 5\penalty0 (1):\penalty0 1--122, 2012.

\bibitem[Cesa-Bianchi and Lugosi(2006)]{cesa2006prediction}
N.~Cesa-Bianchi and G.~Lugosi.
\newblock \emph{Prediction, learning, and games}.
\newblock Cambridge University Press, 2006.

\bibitem[Cesa-Bianchi et~al.(1997)Cesa-Bianchi, Freund, Helmbold, Haussler,
  Schapire, and Warmuth]{1997-acm_how-to-use-expert-advice}
N.~Cesa-Bianchi, Y.~Freund, D.~P. Helmbold, D.~Haussler, R.~Schapire, and
  M.~Warmuth.
\newblock How to use expert advice.
\newblock \emph{Journal of the ACM}, 44(2):\penalty0 427--485, 1997.

\bibitem[Edelman et~al.(2011)Edelman, Jaffe, and Kominers]{edelman2011groupon}
Benjamin Edelman, Sonia Jaffe, and Scott~Duke Kominers.
\newblock To groupon or not to groupon: The profitability of deep discounts.
\newblock \emph{Marketing Letters}, pages 1--15, 2011.

\bibitem[Freund and Schapire(1995)]{freund1995desicion}
Yoav Freund and Robert~E Schapire.
\newblock A desicion-theoretic generalization of on-line learning and an
  application to boosting.
\newblock In \emph{COLT}, pages 23--37, 1995.

\bibitem[Gittins(1979)]{gittins1979bandit}
John~C Gittins.
\newblock Bandit processes and dynamic allocation indices.
\newblock \emph{Journal of the Royal Statistical Society. Series B
  (Methodological)}, pages 148--177, 1979.

\bibitem[Hsu and Lin(2015)]{hsu2015active}
Wei-Ning Hsu and Hsuan-Tien Lin.
\newblock Active learning by learning.
\newblock In \emph{AAAI}, pages 2659--2665, 2015.

\bibitem[Kale(2014)]{kale2014multiarmed}
Satyen Kale.
\newblock Multiarmed bandits with limited expert advice.
\newblock In \emph{COLT}, pages 107--122, 2014.

\bibitem[Langford and Zhang(2007)]{langford2007epochgreedy}
J.~Langford and T.~Zhang.
\newblock The epoch-greedy algorithm for contextual multi-armed bandits.
\newblock In \emph{NIPS}, 2007.

\bibitem[Li et~al.(2010)Li, Chu, Langford, and Schapire]{li2010contextual}
Lihong Li, Wei Chu, John Langford, and Robert~E Schapire.
\newblock A contextual-bandit approach to personalized news article
  recommendation.
\newblock In \emph{WWW}, pages 661--670, 2010.

\bibitem[Littlestone and Warmuth(1994)]{1992_weighted-majority}
Nick Littlestone and Manfred~K. Warmuth.
\newblock The weighted majority algorithm.
\newblock \emph{Info and Computation}, 70(2):\penalty0 212--261, 1994.

\bibitem[Maillard and Munos(2011)]{maillard2011adaptive}
Odalric-{A}mbrym Maillard and R{\'e}mi Munos.
\newblock Adaptive bandits: Towards the best history-dependent strategy.
\newblock In \emph{AISTATS}, pages 570--578, 2011.

\bibitem[McMahan and Streeter(2009)]{mcmahan2009tighter}
H.~B. McMahan and M.~J. Streeter.
\newblock Tighter bounds for multi-armed bandits with expert advice.
\newblock In \emph{COLT}, 2009.

\bibitem[Shalev-Shwartz(2011)]{shalev2011online}
Shai Shalev-Shwartz.
\newblock Online learning and online convex optimization.
\newblock \emph{Foundations and Trends in Machine Learning}, 4\penalty0
  (2):\penalty0 107--194, 2011.

\bibitem[Singla et~al.(2016)Singla, Tschiatschek, and
  Krause]{singla16hemimetric}
Adish Singla, Sebastian Tschiatschek, and Andreas Krause.
\newblock Actively learning hemimetrics with applications to eliciting user
  preferences.
\newblock In \emph{ICML}, 2016.

\bibitem[Slivkins and Upfal(2008)]{slivkins2008adapting}
Aleksandrs Slivkins and Eli Upfal.
\newblock Adapting to a changing environment: the brownian restless bandits.
\newblock In \emph{COLT}, pages 343--354, 2008.

\bibitem[Syrgkanis et~al.(2015)Syrgkanis, Agarwal, Luo, and
  Schapire]{syrgkanis2015fast}
Vasilis Syrgkanis, Alekh Agarwal, Haipeng Luo, and Robert~E Schapire.
\newblock Fast convergence of regularized learning in games.
\newblock In \emph{NIPS}, pages 2971--2979, 2015.

\bibitem[Whittle(1988)]{whittle1988restless}
P.~Whittle.
\newblock Restless bandits: Activity allocation in a changing world.
\newblock \emph{Journal of applied probability}, 1988.

\bibitem[Zinkevich(2003)]{zinkevich03online}
M.~Zinkevich.
\newblock Online convex programming and generalized infinitesimal gradient
  ascent.
\newblock In \emph{ICML}, 2003.

\end{thebibliography}

\clearpage
\appendix
{\allowdisplaybreaks


\section{Proof of Theorem~\ref{hardness_result}}\label{appendix1_theorem1-proof}

In this section, we give a proof of the hardness result by discussing a generic and simple setting in which any forecaster suffers a positive average regret.

{\bfseries The setting.}
Our setting  consists of two experts $\expert_1$ and $\expert_2$.  The first expert $\expert_1$ has two actions given by $\actionSet_1 = \{a_1, a_2\}$, and the second expert $\expert_1$ has only one action given by $\actionSet_2 = \{b\}$. The action set of the forecaster, or algorithm \generalalgo, is given by $\actionSet = \{a_1, a_2, b\}$. The expert $\expert_1$  plays the \hedge algorithm \citep{freund1995desicion}, \emph{i.e.}, the regret rate parameter is $\beta_1 = 0.5$ (see \eqref{eq.noregretdynamics}); the expert $\expert_2$ has only one action to play as in the standard multi-armed bandit with $\beta_2 = 0$ (see \eqref{eq.noregretdynamics}). The forecaster knows parameter $\beta=0.5$ which upper bounds the regret rate of the individual experts.\footnote{In fact, this hardness result holds even when considering a powerful forecaster which knows exactly the learning algorithms used by the experts, and is able to see the losses $\{l^t(a_1), l^t(a_2), l^t(b)\}$ at every time $t \in [T]$.}

{\bfseries Loss sequences.}
Figures~\ref{fig.hardness-example-L1}, ~\ref{fig.hardness-example-L2}, and ~\ref{fig.hardness-example-L3} shows the \textit{cumulative} loss sequences $L_1$, $L_2$, and $L_3$ for three different scenarios---the adversary at $t=0$ uniformly at random picks one of scenarios and uses that loss sequence. These plots show the cumulative losses of the three actions $\actionSet=\{a_1, a_2, b\}$ for three different sequences. For the first scenario with cumulative loss sequences $L_1$ shown in Figure~\ref{fig.hardness-example-L1}, we show in Figure~\ref{fig.hardness-example-instant-losses}  the instantaneous losses of the different actions. Figures~\ref{fig.hardness-example-loss-a1} and~\ref{fig.hardness-example-loss-a2} show the losses of actions for $\expert_1$; ~\ref{fig.hardness-example-loss-b} shows the losses of action for $\expert_2$.


{\bfseries Model specification.}
To fully specify the model and Protocol~\ref{interaction}, we specify now the feedback vector, and the context over time. The context $\context^t$ is constant over time and plays no role in our setting. The experts receive the following feedback when selected: the expert $\expert_1$ would observe the losses $\{l^t(a_1), l^t(a_2)\}$ when $i^t = 1$; and the  expert $\expert_2$ would observe the loss $\{l^t(b)\}$ when $i^t = 2$.

{\bfseries Execution behavior.} We now divide the time horizon $T$ into different slots and discuss the execution behavior of the forecaster and experts over these time slots. Specifically, let us consider the case where the forecaster is facing the sequence $L_1$ (chosen by adversary with probability $\frac{1}{3}$ at $t=0$) with cumulative losses shown in Figure~\ref{fig.hardness-example-L1} and instantaneous losses of the actions shown in Figure~\ref{fig.hardness-example-instant-losses}. 
For a clarity of presentation, we shall use $\Delta = 0.01$ as a constant in rest of the proof below.

\begin{figure*}[!t]
\centering
   \subfigure[$\expert_1$: Action $a_1$ ($L_1$)]{
     \includegraphics[width=0.31\textwidth]{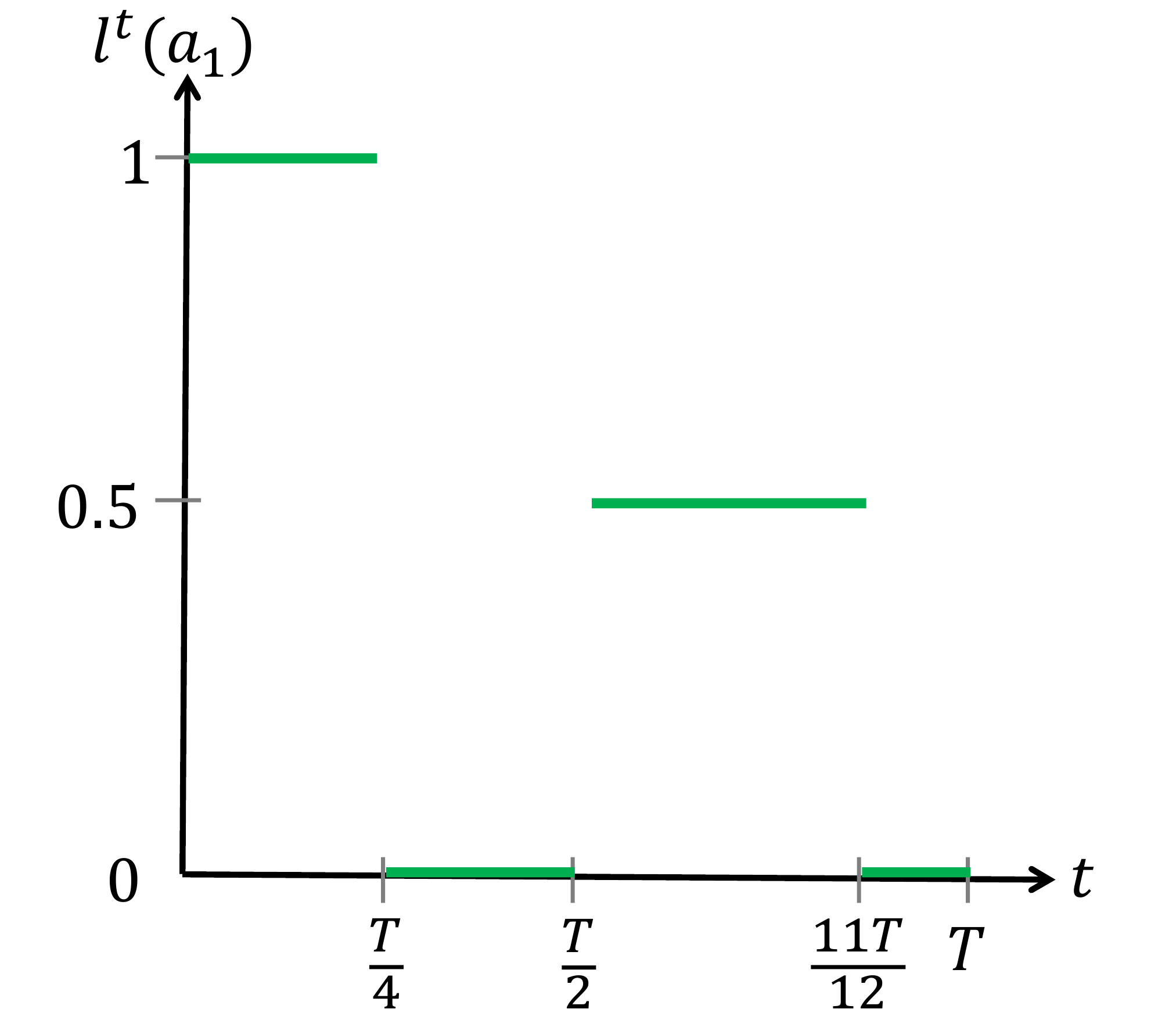}
    \label{fig.hardness-example-loss-a1}
   }
   \subfigure[$\expert_1$: Action $a_2$ ($L_1$)]{
     \includegraphics[width=0.31\textwidth]{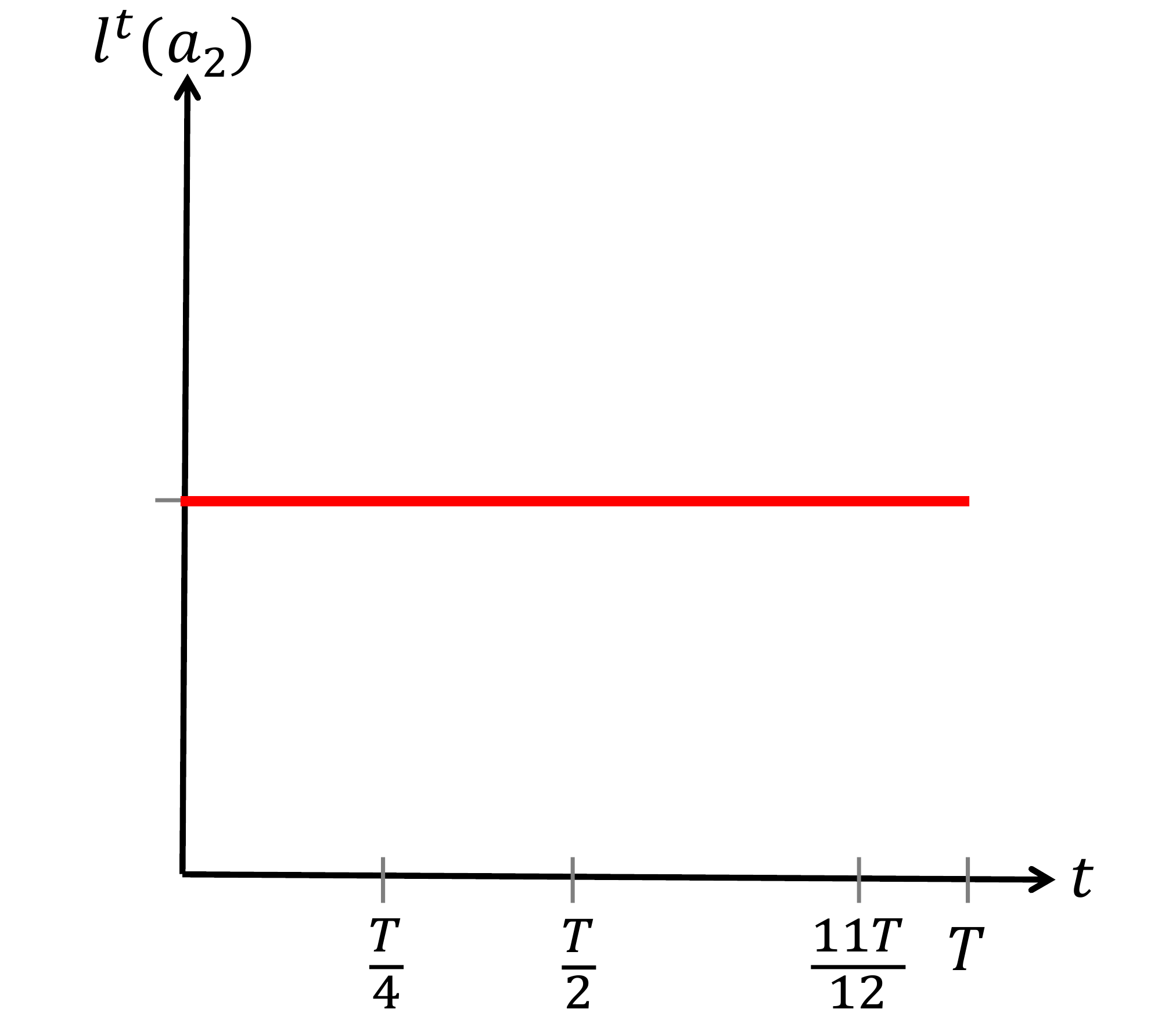}
     \label{fig.hardness-example-loss-a2}
   } 
   \subfigure[$\expert_2$: Action $b$ ($L_3$)]{
     \includegraphics[width=0.31\textwidth]{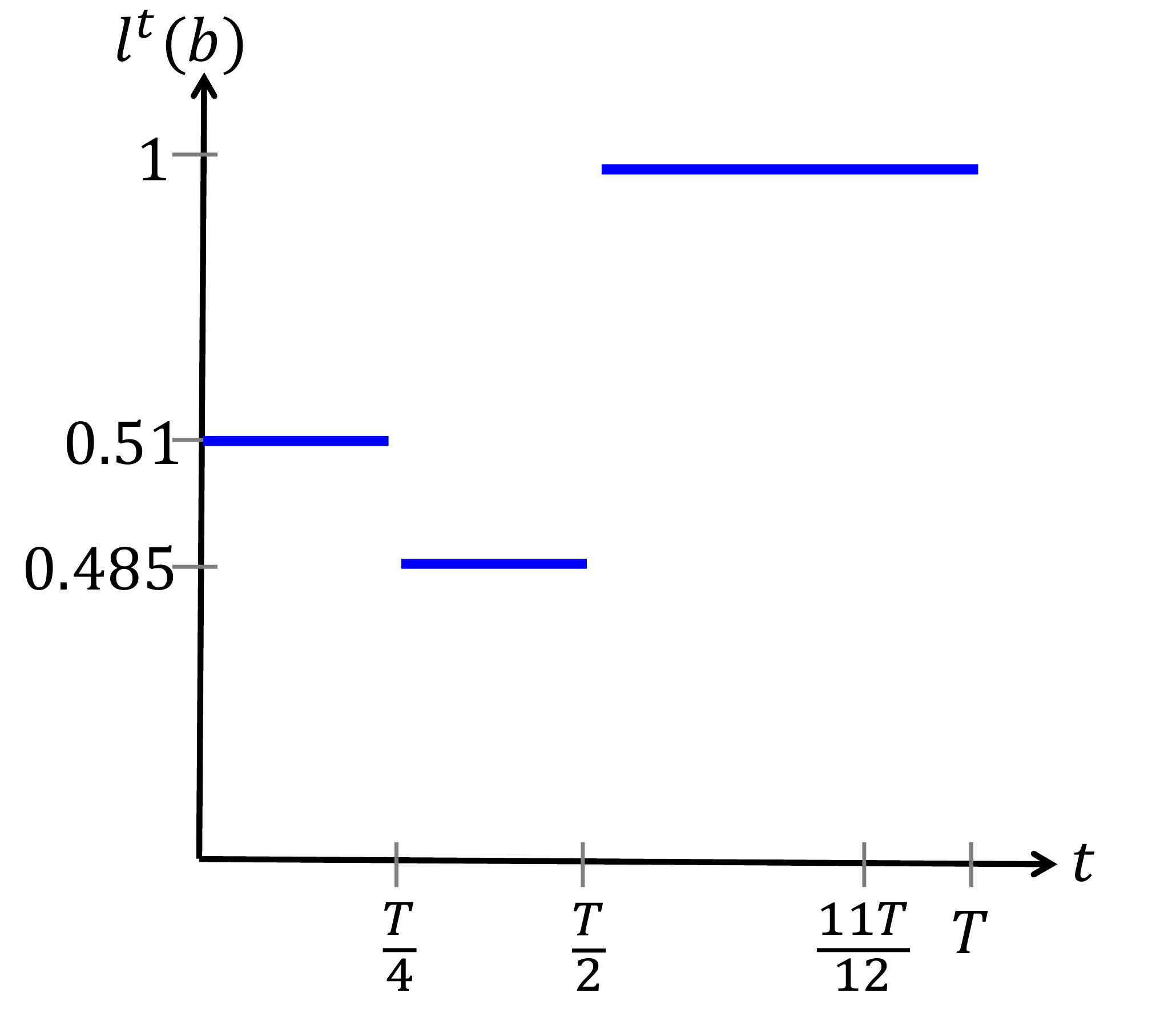}
     \label{fig.hardness-example-loss-b}
   }
\caption{For the first scenario with the cumulative loss sequences $L_1$ shown in Figure~\ref{fig.hardness-example-L1}, here we show the instantaneous losses of the different actions. Figures~\ref{fig.hardness-example-loss-a1} and~\ref{fig.hardness-example-loss-a2} show the losses of actions for $\expert_1$; ~\ref{fig.hardness-example-loss-b} shows the losses of action for $\expert_2$.}
\label{fig.hardness-example-instant-losses}
\end{figure*}

\begin{itemize}
\item $t \in [0, \frac{T}{4}]$: In this time slot, the expert $\expert_1$ would be selected almost surely by the forecaster and the number of times $\expert_2$ would be selected is $\smallObound(T)$. If this is not the case, then this forecaster would suffer positive average regret on the loss sequence $L_3$ for third scenario in Figure~\ref{fig.hardness-example-L3}. \footnote{Note here that the loss sequence $L_3$ is exactly equal to $L_1$ up to time $T/4$. Furthermore, although we have assumed that the setting chosen by the adversary is $L_1$, we should bear in mind that the forecaster  (who can not distinguish between the losses at least up to time $T/4$) should play in a way that it does not suffer positive average regret for $L_3$.} The loss incurred by the forecaster at any time $t$ in this time slot is at least $0.5$.
\item $t \in (\frac{T}{4}, \frac{T}{2}]$: This is the first crucial time slot whereby forecaster's selection strategy would add ``blind spots" to the feedback received by $\expert_1$. The key argument is that in this time slot, the forecaster cannot select expert $\expert_1$ for more than $\frac{T}{6} + o(T)$ timesteps---if this happens, than this forecaster would have a positive average regret on the loss sequence $L_2$ for second scenario in Figure~\ref{fig.hardness-example-L2}. In other words, in this period, the expert $\expert_1$ has missed seeing the feedback for $\frac{T}{12}  - \smallObound(T))$ timesteps. Clearly, the loss incurred by the forecaster at any time $t$ in this time slot is at least $0.5 - \frac{3\Delta}{2}$.
\item $t \in (\frac{T}{2}, \frac{11T}{12}]$: In this time slot, the expert $\expert_1$ would be selected almost surely and the number of times $\expert_2$ would be selected is $\smallObound(T)$. The loss incurred by the forecaster at any time $t$ in this time slot is at least $0.5$.
\item $t \in (\frac{11T}{12}, T]$: This is the second crucial time slot which would lead to the positive average regret for the forecaster. We note that the forecaster still does the ``right" thing in this time slot, \emph{i.e.}, the expert $\expert_1$ would be selected almost surely and the number of times $\expert_2$ would be selected is $\smallObound(T)$. However, the expert $\expert_1$ has missed observing feedback for $(\frac{T}{12} - \smallObound(T))$ time steps in the slot $(\frac{T}{4}, \frac{T}{2}]$. Note also that no coordination is permitted between the forecaster and the experts, and hence, $\expert_1$ is not aware of the time steps that it misses the feedback. As a result, at the start of this time slot, the cumulative loss of action $a_1$ (as perceived by $\expert_1$ based on observed history) is at least $0.5 \cdot \frac{T}{12}$ more than the cumulative loss of action $a_2$ (as perceived by $\expert_1$ based on observed history). The expert $\expert_1$ who is playing $\hedge$ algorithm in our setting would select action $a_2$ almost surely and $a_1$ would be selected $\smallObound(T)$ number of times.
\end{itemize}

{\bfseries Positive average regret.}
Let us now compute the regret of the forecaster when experiencing loss sequence $L_1$ as discussed above. The cumulative loss of the ``best expert in hindsight" is given by that of $\expert_1$ always playing action $a_1$. Based on Figure~\ref{fig.hardness-example-loss-a1}, this is given by:
\begin{align*}
\sum_{t \in [T]} \loss^t(a_1) = 1 \cdot \frac{T}{4} + 0.5 \cdot \big(\frac{11T}{12} - \frac{T}{2}\big)  = T \cdot \big(\frac{1}{2} - \frac{1}{24}\big)
\end{align*}

The cumulative loss of the forecaster as per the execution behavior discussed above can be lower bounded as follows:
\begin{align*}
\sum_{t \in [T]} \loss^t(a^t_{i^t}) &\geq 0.5 \cdot \frac{T}{4} + \big(0.5 - \frac{3 \Delta}{2}\big) \cdot \frac{T}{4} + 0.5 \cdot \big(\frac{11T}{12} - \frac{T}{2}\big) + 0.5 \cdot \big(\frac{T}{12} - \smallObound(T)\big)\\
&= T \cdot \big(\frac{1}{2}  - \frac{3 \Delta}{8} - \frac{\smallObound(T)}{2\cdot T}\big)
\end{align*}

Hence, the total regret of the forecaster is lower bounded by:
\begin{align*}
\regret(\generalalgo, T)  \geq T \cdot \big(\frac{1}{2}  - \frac{3 \Delta}{8} - \frac{\smallObound(T)}{2\cdot T}\big) - T \cdot\big(\frac{1}{2} - \frac{1}{24}\big) = T \cdot \big(\frac{1}{24} - \frac{3 \Delta}{8} - \frac{\smallObound(T)}{2\cdot T} \big)
\end{align*}
Recall that the constant $\Delta = 0.01$, hence the average regret of the forecaster is lower bounded by $\lim_{T \mapsto \infty} \frac{\regret(\generalalgo, T)}{T} \geq \frac{91}{2400}$.

As this sequence $L_1$ is selected by the adversary uniformly at random with probability $\frac{1}{3}$, this means that the forecaster would suffer a positive average regret. As we discussed above, any forecaster which doesn't have the above-mentioned execution behavior in the timeslot $t \in [0, \frac{T}{4}]$ or $t \in (\frac{T}{4}, \frac{T}{2}]$ would suffer a positive average regret for $L_3$ and $L_2$ loss sequences.




\section{Proof of Proposition~\ref{lemma.smooth}}\label{appendix3_lemma1-proof}
\textbf{OCP algorithms.} Assume that and expert $\expert_j$ is performing Online  Convex  Programming (OCP) via greedy projections. We will show that such an algorithm has smooth learning dynamics. Note that OCP has regret of  size $O(\sqrt{T})$ (i.e. $\beta_j = 1/2$). Consider the $\alpha$-OCP algorithm that proceeds according to Algorithm~\ref{alphaocp}. Proving smooth learning dynamics for OCP is equivalent to showing that $\alpha$-OCP suffers a regret of size $O(\sqrt{T/\alpha})$. More precisely, we have the following Lemma.  

\renewcommand{\algorithmcfname}{Algorithm}
\begin{algorithm2e}[h!]
\nl {\bf Problem setting}: {Convex set $\paramSet$; sequence of convex loss functions $f^t : \paramSet \rightarrow \R_{+}$} \\
\nl {\bf Parameters}: Learning rates $\eta^t$ for $t \in [T]$ \\
\nl {\bf Initialize}: $w_0 \in S$ arbitrarily \\
  \ForEach {$t = 1, 2, \ldots, T$}{
	\nl $w^{t+1/2} = w^t - \eta^t B^t z^t $ where:    (i) $z^t \in \partial f^t(\param^t)$, and (ii) random variables $B^t$ are independent Bernoulli with
	parameter $\alpha$ (i.e. $\text{Pr}(B^t =1) = 1 - \text{Pr}(B^t = 0) = \alpha$), and (iii) $\eta_t = 1/ \sqrt{1 + \sum_{\tau =1}^t B^\tau}$\\	
	\vspace{.3cm}	
	\nl $w^{t+1} = \text{Proj}_\paramSet(w^{t+1/2})$
  }
  \caption{$\alpha$-OCP}
  \label{alphaocp} 
\end{algorithm2e}

\begin{lemma}
Let $\norm{S}$ denote the diameter of the convex set $\paramSet$ and $L$ denotes an upper bound on the magnitude of the gradient at any time $t \in \Time$. Then, the expected regret of the $\alpha$-OCP algorithm is given by 
\begin{equation}
\mathbb{E} [\sum_{t=1}^T ( f^t (w^t) - f^t(u)) ] \leq \frac{\norm{S}^2}{2} \cdot  \sqrt{\frac{T}{\alpha}} +  L^2 \cdot \sqrt{\frac{T}{\alpha}}
\end{equation}
where the expectation is w.r.t. the sequence of Bernoulli random variables $B^t$ for $t \in [\Time]$.
\end{lemma}
\begin{proof}
We can equivalently write the updates in the $\alpha$-OCP procedure as follows:
\begin{align*}
    w^{t + 1/2} &= w^t - \eta^t B^t z^t \\
                &= w^t - (\eta^t \cdot \alpha) \cdot (\frac{B^t}{\alpha}) z^t \\
                &= w^t - \tilde{\eta}^t \cdot \tilde{z^t}
\end{align*}
where $\tilde{\eta}^t = (\eta^t \cdot \alpha)$ and $\tilde{z^t} = (\frac{B^t}{\alpha}) z^t$. Note that $\mathbb{E}[\tilde{z^t} | w_{1:t}] = \mathbb{E}[z^t] \in  \partial f^t(w^t)$. We have:
\begin{align}
&\mathbb{E} \bigg[\sum_{t=1}^T ( f^t (w^t) - f^t(u)) \bigg] \\
&= \mathbb{E} \bigg[\sum_{t=1}^T \mathbb{E}\Big[( f^t (w^t) - f^t(u)) | w_{1:t}\Big] \bigg] \\
&\leq \mathbb{E}  \bigg[\sum_{t=1}^T \mathbb{E}\Big[<\partial f^t(w^t), w^t -u> | w_{1:t}\Big]  \bigg] \\
&=\mathbb{E} \bigg[\sum_{t=1}^T \mathbb{E}\Big[<\tilde{z}^t, w^t -u> | w_{1:t}\Big] \bigg] \\
&=\mathbb{E} \bigg[\sum_{t=1}^T \frac{1}{2 \alpha\eta^t }\mathbb{E}\Big[\norm{w^t - w}^2 - \norm{w^{t+1/2} - w}^2 + \eta_t^2 \alpha^2 \norm{\tilde{z^t}}^2 | w_{1:t}\Big] \bigg] \\
&\leq\mathbb{E} \bigg[\sum_{t=1}^T \frac{1}{2 \alpha \eta^t }\mathbb{E}\Big[\norm{w^t - w}^2 - \norm{w^{t+1} - w}^2 + \eta_t^2 \alpha^2 \norm{\tilde{z^t}}^2 | w_{1:t}\Big] \bigg] \\
&= \mathbb{E} \bigg[\sum_{t=1}^T \frac{\norm{w^t - w}^2}{2\alpha \eta^t} - \frac{\norm{w^{t+1} - w}^2}{2 \alpha \eta^t}\bigg]+ \frac{\alpha}{2}\mathbb{E} \bigg[ \sum_{t=1}^T \eta_t || \theta_t || ^2  \bigg] \\
& \leq  \mathbb{E} \bigg[ \frac{\norm{w^1 - w}^2}{2 \alpha \eta^1} - \frac{||w^{T+1} - w||^2}{2\alpha \eta_T}  + \frac 12 \sum_{t=2}^T \norm{w^{t} - w}^2 ( \frac{1}{\alpha \eta^t} - \frac{1}{\alpha \eta^{t-1}} ) \bigg]
+ \frac{\alpha}{2}\mathbb{E} \bigg[ \sum_{t=1}^T \eta_t || \tilde{z}^t  || ^2  \bigg] \\
&\leq ||S||^2 \mathbb{E}\bigg[\frac{1}{2\alpha\eta_T}\bigg] + \frac{\alpha}{2}\mathbb{E} \bigg[ \sum_{t=1}^T \eta_t || \tilde{z}^t  || ^2  \bigg]
\end{align}
Now note that $\mathbb{E}[\eta_t||\tilde{z}^t ||^2] = \mathbb{E}\bigl[\mathbb{E}[\eta_t||\tilde{z}^t ||^2 | w^{1:t}]\bigr]    = \mathbb{E}\bigl[ \eta_t||z^t||^2 / \alpha] \leq L^2\mathbb{E}[\eta_t] / \alpha $. We thus obtain
\begin{align}
&\mathbb{E} \bigg[\sum_{t=1}^T ( f^t (w^t) - f^t(u)) \bigg] 
\leq ||S||^2 \mathbb{E}\bigg[\frac{1}{2\alpha\eta_T}\bigg] + \frac{L^2}{2}\mathbb{E} \bigg[ \sum_{t=1}^T \eta_t  \bigg].
\end{align}
We next recall that $\eta_t = \frac{1}{\sqrt{1+\sum_{\tau=1}^t B^\tau}}$. By using the multiplicative Chernoff bound (as $B^\tau$'s are Bernoulli random variables) we obtain 
$$\text{Pr} ( \eta_t \geq \sqrt{2 / (t  \alpha)}) = \text{Pr} ( \sum_{\tau=1}^t B^\tau \leq t\alpha/2) \leq \exp(- \frac{t\alpha}{12}).$$
Hence, we obtain $\mathbb{E}[\eta_t] \leq \sqrt{2 / (\alpha t)} + \exp(-t\alpha/12)$.
Also, due to concavity of the function $h(x) = \sqrt{x}$, we have that $\mathbb{E}[1/\eta_T] \leq \sqrt{T\alpha}$. We finally obtain
\begin{align*}
\mathbb{E} \bigg[\sum_{t=1}^T ( f^t (w^t) - f^t(u)) \bigg] 
&\leq ||S||^2 \sqrt{ T / \alpha} + L^2 \sqrt{ T/\alpha} +  \sum_{t=1}^T\exp{(-t\alpha/12)} \\
&\leq ||S||^2 \sqrt{ T / \alpha} + L^2 \sqrt{ T/\alpha} +  1 / (1 - \exp{(-\alpha/12)}) \\
&  \leq ||S||^2 \sqrt{ T / \alpha} + L^2 \sqrt{ T/\alpha} + 24/\alpha,
\end{align*}
where the last line is because $1/(1 - \exp(-\alpha/12)) \leq 24/\alpha $ for $\alpha \leq 1$. 
\end{proof}

\textbf{OMD Algorithms.} We now consider  the case that expert $\expert_j$ is performing an algorithm inside the Online Mirror Descent (OMD) family of algorithms. We assume that the algorithm has a regret of order $O(\sqrt{T})$ for any time horizon $T$ (i.e. $\beta_j = 1/2$). We also assume that the algorithm uses the doubling trick. Consider the standard online learning scenario where at any time $t \in [T]$ a convex function $f^t : S \to \mathbb{R}$ is assigned ($S$ is assumed to be a convex region). The proofs proceeds in 3 steps. 

{{\bfseries Step~1. $\alpha$-OMD with a fixed time horizon}\\
We first analyze the algorithm $\alpha$-OMD given in \ref{alphaomd} which is run for a fixed (deterministic) number of steps.

\renewcommand{\algorithmcfname}{Algorithm}
\begin{algorithm2e}[h!]
\nl {\bf Problem setting}: {Convex set $\paramSet$; sequence of convex loss functions $f^t : \paramSet \rightarrow \R_{+}$} \\
\nl {\bf Parameters}: {a link function $g: \R^d \rightarrow \paramSet$; time horizon $\omdT$} \\
\nl {\bf Initialize}: {time $\tau = 1$, auxiliary variable $\theta^\tau = \mathbf{0} \in \R^d$}\\
  \ForEach {$\tau = 1, 2, \ldots, \omdT$}{
	\nl Predict vector $\param^\tau = g(\theta^\tau)$ \\
	\nl Update  $\theta^{\tau+1} = \theta^\tau - B^\tau z^\tau$ where:  (a) $z^\tau \in \partial f^\tau(\param^\tau)$, (b) $B^\tau$ is an independent Bernoulli random variable with parameter $\alpha$ (i.e. $\text{Pr}(B^\tau=1) = 1 - \text{Pr}(B^\tau = 0) = \alpha$).		
  }
  \caption{$\alpha$-OMD}
  \label{alphaomd} 
\end{algorithm2e}

\begin{lemma} \label{omd-fiex}
Let $R$ be a $1/\eta$- strongly convex function over $S$ with
respect to a norm $|| \cdot ||$. Assume that $\alpha$-OMD is run on the sequence with
a link function
\begin{align*}
g(\theta) = \argmax_{w \in S} ( \langle w,\theta \rangle - R(w))
\end{align*}
Furthermore, assume that $f^t$ is $L$-Lipshitz with respect to norm $||\cdot||$. Then
\begin{equation}
\mathbb{E} [\sum_{t=1}^T ( f^t (w^t) - f^t(u)) ] \leq R(u) / \alpha + \\
\eta T L^2 . 
\end{equation}
\end{lemma}
\begin{proof}
For the sake of analysis, we introduce the following slightly modified procedure:
\begin{enumerate}
\item Initialize $\tilde{\theta}^1 = \theta^1 / \alpha$.
\item At time $\tau = 1, 2, \cdots, T$, let $\tilde{w}^\tau = \tilde{g}(\tilde{\theta}^\tau)$, and $\tilde{\theta}^{\tau+1} = \tilde{\theta}^\tau - \tilde{z}^\tau$. Here, we have $\tilde{z}^\tau = \frac{B^\tau}{\alpha} z^\tau$, and the function $\tilde{g}$ is defined as $\tilde{g}(\theta) = \argmax_{w \in S} ( \langle w,\theta \rangle - R(w)/\alpha) $.
\end{enumerate}
It is straight forward to justify for any $\tau \in [T]$ that $\tilde{\theta}^\tau = \theta^\tau / \alpha$ and $\tilde{w}^\tau = w^\tau$. Also note that $\mathbb{E}[ \tilde{z}^\tau  | \tilde{z}^{1:\tau-1} ] = z^\tau \in  \partial f^\tau(\param^\tau)$. Hence, the modified procedure ($\tilde{\theta}^\tau, \tilde{w}^\tau$) is precisely a stochastic OMD procedure with with link function $\tilde{g}$.  By using Theorem 4.1 in \cite{shalev2011online}, $\tilde{w}^\tau = w^\tau$, and the fact that $R(\cdot)/\alpha$ is a $1/(\eta \alpha)$-strongly convex function, we obtain
\begin{align*}
\mathbb{E} [\sum_{t=1}^T ( f^t (w^t) - f^t(u)) ] 
&\leq \sup_{u \in S} R(u)/\alpha + \eta \alpha \sum_{\tau = 1}^T  \mathbb{E} [|| \tilde{z}^\tau ||^2].
\end{align*}
We finally note that 
$$\mathbb{E}[||\tilde{z}^\tau||^2] = \mathbb{E}\bigl[\mathbb{E}[||\tilde{z}^\tau||^2 \, | \, \tilde{z}^{1:\tau-1} ] \bigr] \leq L^2 / \alpha.$$
The result of the Lemma is now immediate. 
\end{proof}

{{\bfseries Step~2. $\alpha$-OMD with a random time horizon}\\ 
From Lemma~\ref{omd-fiex}, for $\eta = O(1/\sqrt{T})$, the algorithm $\alpha$-OMD suffers a $O(\sqrt{T/\alpha})$ regret after any fixed time $T$. Recall now that at any time the algorithm is only given feedback with independent probability $\alpha$. We are assuming that the algorithm used by the expert $\expert_j$ is performing the doubling trick, i.e., it runs in blocks whose size get doubled consecutively and within each block the learning rate is fixed. As a result, after the algorithm receives sufficient feedback to finish a block, it restarts OMD and changes the learning rate for the next block (which has twice the size). In order to analayze the regret suffered in each block, we need to consider a slightly different version of $\alpha$-OMD which stops after a randomly chosen time.

\begin{lemma} [$\alpha$-OMD with a random time horizon] \label{OMD-random}
Assume that we run the $\alpha$-OMD procedure until the time, call it $T_{\rm stop}$, such that following stopping criterion has been fulfilled: 
\begin{equation}
\sum_{\tau = 1}^{T_{\rm stop}} B^\tau = M.
\end{equation}
We the have
\begin{equation}
\mathbb{E} [\sum_{t=1}^{T_{\rm stop}} ( f^t (w^t) - f^t(u)) ]  \leq R(u) / \alpha + \\ \eta M L^2/\alpha + 14 L ||S||\sqrt{M/\alpha^2},
\end{equation}
where $\norm{S}$ denote the diameter of the convex set $\paramSet$ and $L$ denotes an upper bound on the Lipshitz parameter of all the functions $f_t$.
\end{lemma}

\begin{proof}
We can write 
\begin{align*}
& \mathbb{E} [\sum_{t=1}^{T_{\rm stop}} ( f^t (w^t) - f^t(u)) ] \\
& = \mathbb{E} \bigl[\sum_{t=1}^{M/\alpha} ( f^t (w^t) - f^t(u)) \bigr] 
- \biggl( \mathbb{E} [\sum_{t=1}^{M/\alpha} ( f^t (w^t) - f^t(u)) ] - \mathbb{E} [\sum_{t=1}^{T_{\rm stop}} ( f^t (w^t) - f^t(u)) ] \biggr) \\
&\leq  \mathbb{E} [\sum_{t=1}^{M/\alpha} ( f^t (w^t) - f^t(u)) ] 
+  L ||S|| \times  \mathbb{E} [|T_{\rm stop} - M/\alpha|], 
\end{align*}
where the last step follows from the fact that for any two $u,v \in S$ we have $|f(u) - f(v) | \leq L ||S||$. The first term above can be bounded using Lemma~\ref{omd-fiex}. We thus need to upper-bound the expected value of $||T_{\rm stop} - M/\alpha||$. As $B^\tau$'s are Bernoulli($\alpha$) random variables, we expect that 
$T_{\rm stop}$ concentrates around $M/\alpha$. By using the multiplicative  Chernoff bound we have
\begin{align*}
\text{Pr} ( T_{\rm stop} \geq  M/\alpha + \beta) & = \text{Pr} (\sum_{\tau = 1}^{M/\alpha + \beta} B^\tau \leq M) \\
& \leq \text{Pr} (\sum_{\tau = 1}^{M/\alpha + \beta} B^\tau \leq (M + \alpha \beta)(1 - \frac{\alpha \beta}{M + \alpha \beta}) ) \\
& \leq \exp( - \frac{(\alpha \beta)^2} {3 (M + \alpha \beta)}).
\end{align*}
Similarly, we can show that
\begin{align*}
\text{Pr} ( T_{\rm stop} \leq  M/\alpha - \beta) \leq \exp( - \frac{(\alpha \beta)^2} {3 (M - \alpha \beta)}).
\end{align*}
We thus obtain,
\begin{align*}
\mathbb{E} [ |T_{\rm stop} - M/\alpha|] & \leq\sum_{j=0}^{M/\alpha} \text{Pr} (T_{\rm stop} \leq M/\alpha - j) +  \sum_{j=0}^{\infty} \text{Pr}(T_{\rm stop} \geq M/\alpha + j) \\
&\leq \sum_{j=0}^{M/\alpha} \exp(-\frac{(\alpha j)^2} {3 (M - \alpha j)} ) +  \sum_{j=0}^{\infty} \exp(- \frac{(\alpha j)^2} {3 (M + \alpha j)}) \\
& \leq 2 \sum_{j=0}^{\infty} \exp(- \frac{(\alpha j)^2} {3 (M + \alpha j)}) \\
& \leq 2 \sum_{k=0}^\infty \sqrt{M/\alpha^2} \exp(-\frac{M k^2}{3(M+k\sqrt{M})}) \\
&\leq 2 \sqrt{M/\alpha^2} \sum_{k=0}^\infty  \exp(-\frac{k}{6}) \\
& \leq 14 \sqrt{M/\alpha^2}.
\end{align*}
\end{proof}

{{\bfseries Step~3. Putting things together}\\  When the algorithm run by $\expert_j$ is using the doubling trick, for each round (with a block of size $M$), the algorithm needs to be given $M$ feedbacks (from the forecaster) until it switches to the next round (i.e. it doubles the block-length and restarts the algorithm).  Therefore, due to the fact that feedback from the forecaster is sent with independent probability $\alpha$,  the total time needed for the algorithm to switch to the next round is as the one given in Lemma~\ref{OMD-random}. As a result, the regret suffered in the current round is upper-bounded $O(\sqrt{M/\alpha^2})$ (from Lemma~\ref{OMD-random}). Note here that the time spent in each round to give $M$ feedbacks to the algorithm (i.e. $T_{\rm stop}$ in Lemma~\ref{OMD-random}) is roughly $M/\alpha$.  Now, assume that the total time taken by the algorithm is $T$. The algorithm (which is given feedback with probability $\alpha$ and plays according to the doubling trick) will be given feedback in $T\alpha$ time units (on average). As a result, it is not hard to see that total regret (after summing up over all the rounds played by the algorithm and using Jensen) becomes $\sqrt{T/\alpha}$. Hence, the proposition is proved also for the OMD algorithms with regret $O(\sqrt{T})$.

\section{Proof of Theorem~\ref{thm.regretbounds}}\label{appendix1_theorem1-proof}
In this section, we provide the proof of Theorem~\ref{thm.regretbounds} for the no-regret guarantees of our algorithm \algo. We follow a step by step approach, beginning with the bounds on external regret of \algo.

{{\bfseries Step~1. Bounds on external regret of \algo}}\\
By directly using the bounds of \expthree algorithm, cf. Theorem 3.1 from \cite{auer2002nonstochastic}, we can state the following bounds on the external regret of our algorithm \algo against any  expert $\expert_k$ where $k \in [\num]$. Note that these bounds given by the \emph{external} regret are only w.r.t. to the post hoc sequence of actions performed and losses observed during the execution of the algorithm

\begin{align}
\sum_{t=1}^{\Time}\E\bigg[\loss^t\Big(\pi_{i^t}(\context^t, \historySet^t_{{i^t},\algo})\Big)\bigg] - \E\bigg[\sum_{t=1}^{\Time} \loss^t\Big(\pi_k(\context^t, \historySet^{t}_{k, \algo})\Big)\bigg] \leq c \cdot \eta \cdot \Time + \frac{(\log\num) \cdot \num}{\eta}  \label{th3.step1}
\end{align}
where $c$ is a constant given by $c = e - 1$.

\vspace{2mm}
{{\bfseries Step2. No-regret and smooth learning dynamics of the experts}}\\
We note that during the execution of \algo in Algorithm~\ref{mainalgo}, we have sparse feedbacks whereby the experts receive feedback instances sporadically at rate defined by $\alpha = \frac{\eta}{\num}$, \emph{cf.} Section~\ref{sec.alg}. Hence, by definition, we have $\historySet^t_{j,\algo} \equiv \historySet^t_{j, \alpha} \ \forall j \in [\num]$ where $\alpha=\frac{\eta}{\num}$, \emph{cf.} Section~\ref{sec.alg}. 
By definition, the no-regret smooth learning dynamics of the expert $k$ guarantees:
\begin{align}
\E\bigg[\sum_{t=1}^{\Time} \loss^t\Big(\pi_k(\context^t, \historySet^{t}_{k, \algo})\Big)\bigg] - \E\bigg[\sum_{t=1}^{\Time} \loss^t\Big(\pi_k(\context^t, \historySet^{T}_{k, 1})\Big)\bigg] &\leq \bigObound\Big(\Time \cdot \Big(\alpha \cdot \Time\Big)^{\regretRate_k - 1} \Big)  \notag \\
&=   \bigObound\Big(\frac{\Time^{\regretRate_k} \cdot \num^{1 - \regretRate_k}}{\eta^{1 - \regretRate_k}}\Big), \label{th3.step2}
\end{align}
where $\regretRate_k$ is the parameter defining the rate of growth of regret, \emph{cf.} Section~\ref{sec.alg}.

\vspace{2mm}
{{\bfseries Step3. Putting it together}}\\
Let us rewrite the regret of the algorithm, copying from Equation~\ref{eq.objective}:
\begin{align}
\regret(\Time, \algo) \coloneqq \sum_{t=1}^{\Time}\E\bigg[\loss^t\Big(\pi_{i^t}(\context^t, \historySet^t_{{i^t},\algo})\Big)\bigg] - \min_{j \in [\num]} \E\bigg[\sum_{t=1}^{\Time} \loss^t\Big(\pi_j(\context^t, \historySet^{T}_{j, 1})\Big)\bigg] \label{th3.step3.a}
\end{align}

Combining Eq.\ref{th3.step1} and Eq.\ref{th3.step2} from above, and using the definition of $\regret$ from Equation~\ref{th3.step3.a}  above, we get:
\begin{align}
\regret(\Time, \algo) \leq \bigObound\Big(\eta \cdot \Time + \frac{(\log\num) \cdot \num}{\eta} +  \frac{\Time^{\regretRate_k} \cdot \num^{1 - \regretRate_k}}{\eta^{1 - \regretRate_k}}\Big) \label{th3.step3.b}
\end{align}

{{\bfseries Step4. Optimizing $\eta$}}\\
%
%
Next, we will optimize the value of $\eta$ in terms of $\Time$ and $\num$. Note that $\expert_k$ above corresponds to any expert. Hence, let us set $k = j^*$ where $j^*$ corresponds to the best expert $\expert_{j^*}$ that we want to compete against. As per assumptions of the theorem, the  best expert indeed has no-regret smooth learning dynamics with $\regretRate_{j^*} \in [0, 1]$. Stating this in terms $k=j^*$, we can write down the regret as follows:
\begin{align}
\regret(\Time, \algo) \leq \bigObound\Big(\eta \cdot \Time + \frac{(\log\num) \cdot \num}{\eta} +  \frac{\Time^{\regretRate_{j^*}} \cdot \num^{1 - \regretRate_{j^*}}}{\eta^{1 - \regretRate_{j^*}}}\Big) \label{th3.step4.a}
\end{align}

However, note that algorithm doesn't know $\regretRate_{j^*}$ and hence cannot directly optimize the value of $\eta$. As per the theorem statement, the \algo is invoked with input $\regretRate \in [0, 1]$ such that $\regretRate \geq \regretRate_{j^*}$.  



\vspace{2mm}
{{\bfseries Step4.1 Optimizing $\eta$ for known $\regretRate_{j^*}$}, \emph{i.e.}, $\regretRate = \regretRate_{j^*}$}\\
To begin with, let us first optimize $\eta$ for case when $\regretRate = \regretRate_{j^*}$. 
In order to find the optimal dependency of $\eta$ on $\Time$, we set $\eta \sim \Time^{-z}$, and the value $z$ will be found to minimize the external regret. By this choice of $\eta$,  the following terms stated as the powers of $\Time$ appear in \eqref{th3.step4.a}:
\begin{align}
 \{\Time^{1-z}, \Time^{z},  \Time^{z + \regretRate_{j^*}\cdot(1 - z)}\}
\end{align}
Solving for optimal value of $z$ to minimize the power of $\Time$ in the leading term, we get $z = \frac{1 - \regretRate_{j^*}}{2 - \regretRate_{j^*}}$.
%
%

Next, we find the optimal dependency of $\eta$ on $\num$. Note that, when $\regretRate = 0$, we have optimal dependency of $\eta$ on $\num$ as $(\num \cdot \log(\num))^{\frac{1}{2}}$. In general, the optimal dependency of $\eta$ to $N$ can be found by setting $\eta \sim \num^{z}$, which gives us from \eqref{th3.step4.a} the following terms stated as the powers of $\num$, where only the leading terms w.r.t. $\Time$ are kept:
\begin{align}
 \{\num^{z},  \num^{(1-\regretRate_{j^*})\cdot(1 - z)}\}
\end{align}
Solving for optimal value of $z$ to minimize the power of $\Time$ in the leading term, we get $z = \frac{1 - \regretRate_{j^*}}{2 - \regretRate_{j^*}}$.

For any $\regretRate_{j^*} \in [0, 1]$, we can thus write the optimal value of $\eta$ as:
\begin{align}
\eta = \Time^{-\frac{1 - \regretRate_{j^*}}{2 - \regretRate_{j^*}}} \cdot \num^{\frac{1 - \regretRate_{j^*}}{2 - \regretRate_{j^*}}} \cdot (\log\num)^{(\frac{1}{2} \cdot\indfunc_{\{\regretRate_{j^*} = 0\}})}
\end{align}
By keeping only the leading term of $\Time$, we can write this as follows:
\begin{align}
\regret(\Time, \algo) \leq  \bigObound\big(\Time^{\frac{1}{2 - \regretRate_{j^*}}} \cdot \num^{\frac{1}{2 - \regretRate_{j^*}}} \cdot (\log\num)^{(\frac{1}{2} \cdot\indfunc_{\{\regretRate_{j^*} = 0\}})} \big)
\end{align}

{{\bfseries Step4.2 Optimizing $\eta$ for unknown $\regretRate_{j^*}$}, \emph{i.e.}, $\regretRate \geq \regretRate_{j^*}$}\\
When $\regretRate_{j^*}$ is not known exactly,  and $\regretRate$ only upper bounds $\regretRate_{j^*}$, we can still optimize $\eta$ w.r.t. $\regretRate$ to get the same $\eta$ as stated above, replacing $\regretRate_{j^*}$ by $\regretRate_{j}$ (note that $1/(2-\beta)$ is increasing in $\beta$).
By keeping only the leading term of $\Time$, we can write the regret as follows:
\begin{align}
\regret(\Time, \algo) \leq  \bigObound\big(\Time^{\frac{1}{2 - \regretRate}} \cdot \num^{\frac{1}{2 - \regretRate}} \cdot (\log\num)^{(\frac{1}{2} \cdot\indfunc_{\{\regretRate = 0\}})} \big)
\end{align}

This gives us the desired bound stated in Theorem~\ref{thm.regretbounds}.

}

\end{document}